\documentclass[accepted]{uai2026} 
                        
\usepackage[american]{babel}

\usepackage{natbib} 
\usepackage{mathtools} 
\usepackage{booktabs} 
\usepackage{tikz} 

\usepackage[dvipsnames]{xcolor}
\hypersetup{
    colorlinks=true,
    citecolor=RoyalBlue
}

\usepackage[capitalize,noabbrev]{cleveref}

\usepackage{enumitem}
\usepackage{multirow}
\usepackage{makecell}
\usepackage{algorithm}
\usepackage{algorithmic}
\usepackage{pifont}
\usepackage{microtype}
\usepackage{graphicx}
\usepackage{booktabs} 
\usepackage{subcaption}

\usepackage{amsmath}
\usepackage{amssymb}
\usepackage{mathtools}
\usepackage{amsthm}
\usepackage{bm}

\theoremstyle{plain}
\newtheorem{theorem}{Theorem}[section]
\newtheorem{proposition}[theorem]{Proposition}
\newtheorem{lemma}[theorem]{Lemma}

\theoremstyle{definition}
\newtheorem{definition}[theorem]{Definition}

\theoremstyle{remark}

\title{Remedying Coarsening-Based GNN Training under Heterophily via Adaptive Complementary Enhancement}

\newcommand{\corrauthor}{\textsuperscript{\dag}}

\author[1]{\href{mailto:paskardli@outlook.com}{Guoming Li}{}}
\author[2]{Jian Yang}
\author[3]{Xukun Wang}
\author[4]{Zixiao Wang}
\author[5]{Shangsong Liang}
\author[1]{\href{mailto:yifanc@hkbu.edu.hk}{Yifan Chen\corrauthor}{}}

\affil[1]{Department of Computer Science, Hong Kong Baptist University, Hong Kong SAR, China}
\affil[2]{School of Information, Renmin University of China, Beijing, China}
\affil[3]{Academy of Mathematics and Systems Science, Chinese Academy of Sciences, Beijing, China}
\affil[4]{Mohamed bin Zayed University of Artificial Intelligence, Abu Dhabi, UAE}
\affil[5]{School of Computer Science, Sun Yat-sen University, Guangzhou, China}

\begin{document}

\hypersetup{urlcolor=black}
\maketitle
\hypersetup{urlcolor=magenta}

\begingroup
\renewcommand{\thefootnote}{\fnsymbol{footnote}}
\footnotetext[2]{Corresponding author}
\endgroup

\begin{abstract}
Coarsening-based training for graph neural networks (GNNs), i.e.\ training on coarsened graphs rather than the original large ones, has become a promising direction for scaling GNNs to massive graphs. 
However, prior work has been evaluated almost exclusively on \textit{homophilic} graphs, leaving the more challenging \textit{heterophilic} settings underexplored. 
We show, both empirically and theoretically, that existing coarsening-based training methods suffer significant performance degradation on heterophilic graphs due to inevitable loss of graph information during coarsening.
To address this, we propose {\bf A}daptive {\bf C}omplementary {\bf E}nhancement, a plug-and-play, model-agnostic strategy that reintegrates the information discarded in coarsening: 
ACE learns a projector for re-constructing
original node features and applies \textit{anisotropic structural regularization} to embed local heterophily.
We further adopt \textit{homoscedastic uncertainty weighting} to adaptively balance the combined training objective of primary coarsened-graph training loss and full-graph auxiliary loss with augmented node features re-constructed by the heterophily-aware projector. 
Extensive experiments show that ACE drives consistent gains on heterophilic benchmarks while preserving competitive results on homophilic graphs with minimal computational overhead. 
Code is available at the GitHub repository:~\url{https://github.com/vasile-paskardlgm/ACE}.
\end{abstract}

\section{Introduction}
\label{sec:intro}

Graph neural networks (GNNs) have emerged as powerful tools for capturing structural information from graph signals, achieving prominent performance across a broad spectrum of structured and graph-based applications~\citep{GNNFoundationsApplications}. 
However, as real-world graphs scale to millions of nodes and billions of edges, the computational complexity of mainstream propagation algorithms results in formidable training costs~\citep{gnn-survey}.
In response, various scalable training strategies have been proposed, including \textit{model-centric} graph sampling~\citep{sampling-GNN-training} and \textit{data-centric} graph reduction~\citep{graph-reduction-survey}.

Among data-centric methods, \textbf{graph coarsening} has emerged as a widely adopted paradigm~\citep{GC-scal,Gcoarse-7-FGC,UGC,SGBGC,GC-GCNonly-convmatch}. By training GNNs on a reduced graph to produce predictions for the original scale, these methods achieve significant efficiency. Unlike \textit{graph condensation}~\citep{Gcond}, coarsening is \textit{model-agnostic}, allowing a single coarsened graph to support a broad family of GNN architectures~\citep{GC-scal,UGC}.



Despite this empirical success, existing coarsening methods have been developed almost exclusively for \textbf{homophilic} graphs, leaving more challenging \textbf{heterophilic} graphs (distinct from \textit{heterogeneous} graphs; \citealt{heterogeneous-graph-survey}) largely unexplored~\citep{heterophily-gnn-survey-2}. 
In homophilic graphs, neighboring nodes tend to share similar labels or features, whereas this assumption breaks down in heterophilic graphs; 
therein neighbors often possess different labels or representations. 

Our preliminary analysis (in Section~\ref{sec:preliminaries-challenge}) reveals that coarsening-based training suffers a much greater performance drop on heterophilic graphs relative to full-graph training, while the degradation on homophilic graphs is notably smaller. 
We turn to explain this disparity through the lens of mutual information, showing that graph heterophily significantly exacerbates the \textit{mutual information gap} between models trained via coarsening and those trained on the full graph. These results suggest that heterophily poses a fundamental challenge to coarsening-based training.

In response, this work proposes {\bf A}daptive {\bf C}omplementary {\bf E}nhancement (ACE), a plug-and-play, model-agnostic framework for enhancing coarsening-based GNN training under heterophily. The central principle of ACE is to reintegrate the graph information discarded during coarsening via a heterophily-aware auxiliary loss.
In particular, ACE begins by learning a refined \textbf{projector} for re-constructing original node features with averaged supernode feature;
this learning process is further features \textit{anisotropic structural regularization} (c.f.\ Eq.~\eqref{eq:ASR}) for identifying heterophily. 
This regularization stems from anisotropic graph diffusion and were primarily adapted to supervised GNN training settings~\citep{graph-anisotropic-diffusion-1,graph-anisotropic-diffusion-4}; 
here, we formulate it as an unsupervised data augmentation mechanism to embed local heterophily into the reconstructed features as well as the projector.


The resulting projector lifts GNN predictions from the coarsened graph back to the full-graph scale to define an auxiliary loss with respect to the full-graph labels. 
To further stabilize training, we adopt \textit{homoscedastic uncertainty}~\citep{homoscedastic-uncertainty} to adaptively weigh the auxiliary objective against the primary coarsening loss. 
ACE's modularity enables seamless integration into existing pipelines, where our empirical studies verify consistent performance gains. 
To summarize, our main contributions are as follows:
\begin{itemize}
[leftmargin=15pt,parsep=2pt,itemsep=2pt,topsep=2pt]
    \item We recognize the heterophily challenge in coarsening-based GNN training, and 
    provide empirical evidence and theoretical explanations highlighting this underexplored limitation.
    \item We propose {\bf A}daptive {\bf C}omplementary {\bf E}nhancement, a plug-and-play, model-agnostic appraoch that augments coarsening-based training with an heterophily-aware auxiliary loss.
    \item We conduct extensive experiments on large-scale heterophilic and homophilic benchmarks, demonstrating that ACE consistently improves coarsening training by a notable margin.
\end{itemize}

\section{Preliminaries and Challenges}\label{sec:preliminaries}
We first introduce key notations and background knowledge, then identify the heterophily challenge in existing coarsening training via empirical and theoretical analysis.

\subsection{Notations and Preliminaries}\label{sec:preliminaries-background}

\paragraph{Notations.} 
Let $\mathcal{G} = (A, X)$ represent a graph, where $A \in \{0,1\}^{n \times n}$ denotes the adjacency matrix and $X \in \mathbb{R}^{n\times d}$ is the node features/signals. 
The normalized graph Laplacian is defined as $L = I-D^{-\frac{1}{2}} A D^{-\frac{1}{2}}$~\citep{spectralgraphtheory}, where $I$ is the identity matrix and $D$ is the degree matrix. 
Additionally, we denote $Y \in \{0,1\}^{n \times c}$ as the one-hot encoded label matrix, where $c$ is the number of node classes. 
On graph coarsening, we follow prior literature~\citep{Gcoarse_survey,Gcoarse_survey-2},
and let $\mathcal{C}_{1}, \mathcal{C}_{2}, ..., \mathcal{C}_{n^{\prime}}$ be a partition of the node set into $n^{\prime} < n$ disjoint clusters, where each cluster corresponds to a \textit{supernode} in the coarsened graph.
The partition matrix $P \in \{0,1\}^{n^{\prime} \times n}$ satisfies $P_{i,j} = 1$ if node $v_{j} \in \mathcal{C}_{i}$. 
The coarsened graph $\mathcal{G}^{\prime} = (A^{\prime}, X^{\prime})$ is given by: $A^{\prime} = P A P^{T}$, $X^{\prime} = C^{-1} P X$, and $Y^{\prime} = C^{-1} P Y$, where $C = \text{diag}(\vert \mathcal{C}_{1} \vert, \vert \mathcal{C}_{2} \vert, ..., \vert \mathcal{C}_{n^{\prime}} \vert)$.

\paragraph{Coarsening-based GNN Training.} 
Following prior works (e.g.,~\citep{GC-scal,SGBGC}), coarsening training seeks to improve efficiency by training on a coarsened graph $\mathcal{G}^{\prime}$, while aiming to retain comparable downstream performance to models trained on the full graph $\mathcal{G}$. 
Formally, this paradigm can be defined as follows:
\begin{definition}[Coarsening-based GNN Training]
\label{definition:problem}
Let $f(; \Theta)$ denote a GNN with parameter set $\Theta$ and a softmax output layer. 
Coarsening-based GNN training minimizes the task loss on the coarsened graph: $\mathcal{L}(f(A^{\prime}, X^{\prime}; \Theta), Y^{\prime})$, which serves as a surrogate objective for the original graph loss: $\mathcal{L}(f(A, X; \Theta), Y)$.

In this work, we focus on the \textit{node classification} task, where $\mathcal{L}$ instantiated as the \textit{cross-entropy} loss, following prior studies~\citep{GC-scal,Gcoarse-7-FGC,UGC,SGBGC}.
\end{definition}

\paragraph{Graph with Heterophily.}
Heterophilic graphs—distinct from \textit{heterogeneous} graphs~\citep{heterogeneous-graph-survey}—have become a central focus in modern node-classification research~\citep{heterophily-gnn-survey,heterophily-gnn-survey-2}.
Unlike homophilic graphs, where neighboring nodes typically share similar labels or attributes, heterophilic graphs connect dissimilar nodes, violating key assumptions behind message-passing GNNs and often exacerbating \textit{over-smoothing} effects~\citep{oversmoothing,hetero-smooth}. 
Such structures commonly arise in real-world networks (e.g., financial, communication, and certain molecular graphs) and pose distinct challenges for GNN architectures~\citep{dataset6-large-hetero,dataset8-small-hetero}.
Despite their importance, the impact of heterophily has remained largely unexamined in coarsening-based GNN training, leaving a notable gap that we explore in this work.

\paragraph{Dirichlet Energy on Graphs.}
Let $E$ be the edge set of graph $\mathcal{G}$. 
The Dirichlet energy of node features $X$ is typically defined as
\begin{equation}
\label{eq:dirichlet-energy}
\mathcal{E}(\mathcal{G}) = \frac{1}{2}\sum_{(i,j)\in E}A_{ij}\|\frac{X_{i}}{\sqrt{D_{ii}}}-\frac{X_{j}}{\sqrt{D_{jj}}}\|_2^2 = \text{trace}(X^{\top}LX)\ .
\end{equation}
Dirichlet energy serves as a fundamental descriptor of smoothness over the graph and is central to understanding diffusion, propagation, and convolution behavior on graphs~\citep{graph-anisotropic-diffusion-6,graph-diffusion-survey}. 
This concept closely tied to the aforementioned key phenomena: heterophily, where label signals exhibit high Dirichlet energy~\citep{fLode}, and over-smoothing, where repeated graph propagation drives the energy of node representations toward zero~\citep{oversmoothing-energy}.

\subsection{Heterophily Challenge}\label{sec:preliminaries-challenge}

\paragraph{Empirical Observations.} We examine how existing coarsening-based GNN training methods behave under heterophily. 
Specifically, we evaluate three representative approaches—SCAL~\citep{GC-scal}, FGC~\citep{Gcoarse-7-FGC}, and SGBGC~\citep{SGBGC}—using a 10\% coarsening ratio across large heterophilic graphs (Genius, Gamers, Wiki~\citep{dataset6-large-hetero}) and homophilic graphs (Ogbn-arxiv, Ogbn-products~\citep{dataset5-ogb}). 
The experiments are conducted with three backbone GNNs: GCN~\citep{GCN}, LINKX~\citep{dataset6-large-hetero}, and GloGNN~\citep{glognn++}, where the latter two are explicitly designed for heterophilic settings. 
Full details follow Section~\ref{sec:experiments-setup}.

\begin{table*}[!th]
\caption{Impact of heterophily on coarsening-based training. ``-'' denotes full-graph training, and $\Delta\downarrow$ indicates the average performance drop of coarsening training relative to full-graph training for each dataset.}
\label{table-numerical-heterophily-performance}
\centering
\setlength{\tabcolsep}{5pt}
\resizebox{\textwidth}{!}{
\begin{tabular}{ccccc|cccc|cccc|c}
\toprule
Backbone & \multicolumn{4}{c|}{GCN}      & \multicolumn{4}{c|}{LINKX}    & \multicolumn{4}{c|}{GloGNN}   & \multirow{2}{*}{$\Delta\downarrow$} \\ \cmidrule(lr){2-13}
Method   & -     & SCAL  & FGC   & SGBGC & -     & SCAL  & FGC   & SGBGC & -     & SCAL  & FGC   & SGBGC &                                     \\ \midrule
Genius   & 87.42 & 67.47 & 70.13 & 71.22 & 90.77 & 60.62 & 65.49 & 62.68 & 90.66 & 59.47 & 63.81 & 67.33 & \textbf{26.26}                               \\ \cmidrule(lr){2-14} 
Gamers   & 62.18 & 47.39 & 40.14 & 44.62 & 66.06 & 50.13 & 42.48 & 47.60 & 66.19 & 46.28 & 45.92 & 47.15 & \textbf{19.07}                               \\ \cmidrule(lr){2-14} 
Pokec    & 75.45 & 50.44 & 53.81 & 55.48 & 82.04 & 56.48 & 57.63 & 54.18 & 83.00 & 61.58 & 66.91 & 63.77 & \textbf{22.35}                               \\ \midrule
Arxiv    & 71.74 & 63.42 & 65.52 & 64.12 & 69.54 & 57.71 & 59.38 & 61.55 & 72.68 & 57.09 & 61.73 & 61.24 & 7.79                                \\ \cmidrule(lr){2-14} 
Products & 75.64 & 68.37 & 71.28 & 71.96 & 74.59 & 62.49 & 67.52 & 64.93 & 77.48 & 66.53 & 70.46 & 69.11 & 6.50                                \\ \bottomrule
\end{tabular}}
\end{table*}

Results in Table~\ref{table-numerical-heterophily-performance} reveals a striking trend: \textbf{all} coarsening-based methods suffer substantially larger performance degradation on heterophilic graphs compared to homophilic ones—often nearly \textbf{four times} larger—and this occurs consistently across all backbone models, including heterophily-specialized architectures (LINKX and GloGNN).

Notably, these heterophily-oriented models exhibit even larger performance drops, underscoring that the difficulty stems not from the model itself but from a more \textbf{fundamental limitation of coarsening-based training under heterophilic settings}. 
More empirical results supporting this observation are presented in the main experiments in Section~\ref{sec:experiments} and Appendix~\ref{app:expdetails-moreexp}.

\paragraph{Theoretical Explanation.} The empirical results on heterophilic graphs reveal a pronounced and consistent performance gap between models trained on the coarsened graph and those trained on the full graph, indicating that the degradation arises from the coarsening process itself. 
Motivated by this observation, we study the heterophily challenge from an information-theoretic perspective to understand how and why this gap arises.

Let $\mathcal{G} \setminus \mathcal{G}^{\prime}$ denote the \textit{discarded graph information} during coarsening, consisting of the intra-supernode structural details (illustrated in Figure~\ref{fig:coarsening-residual}). Using this notion, we characterize the discrepancy between the full-graph–trained model and the coarsened-graph–trained model by quantifying the mutual information between $\mathcal{G} \setminus \mathcal{G}^{\prime}$ and the label $Y$, leading to the following proposition:

\begin{figure}[!t]
  \centering
  \includegraphics[width=\linewidth]{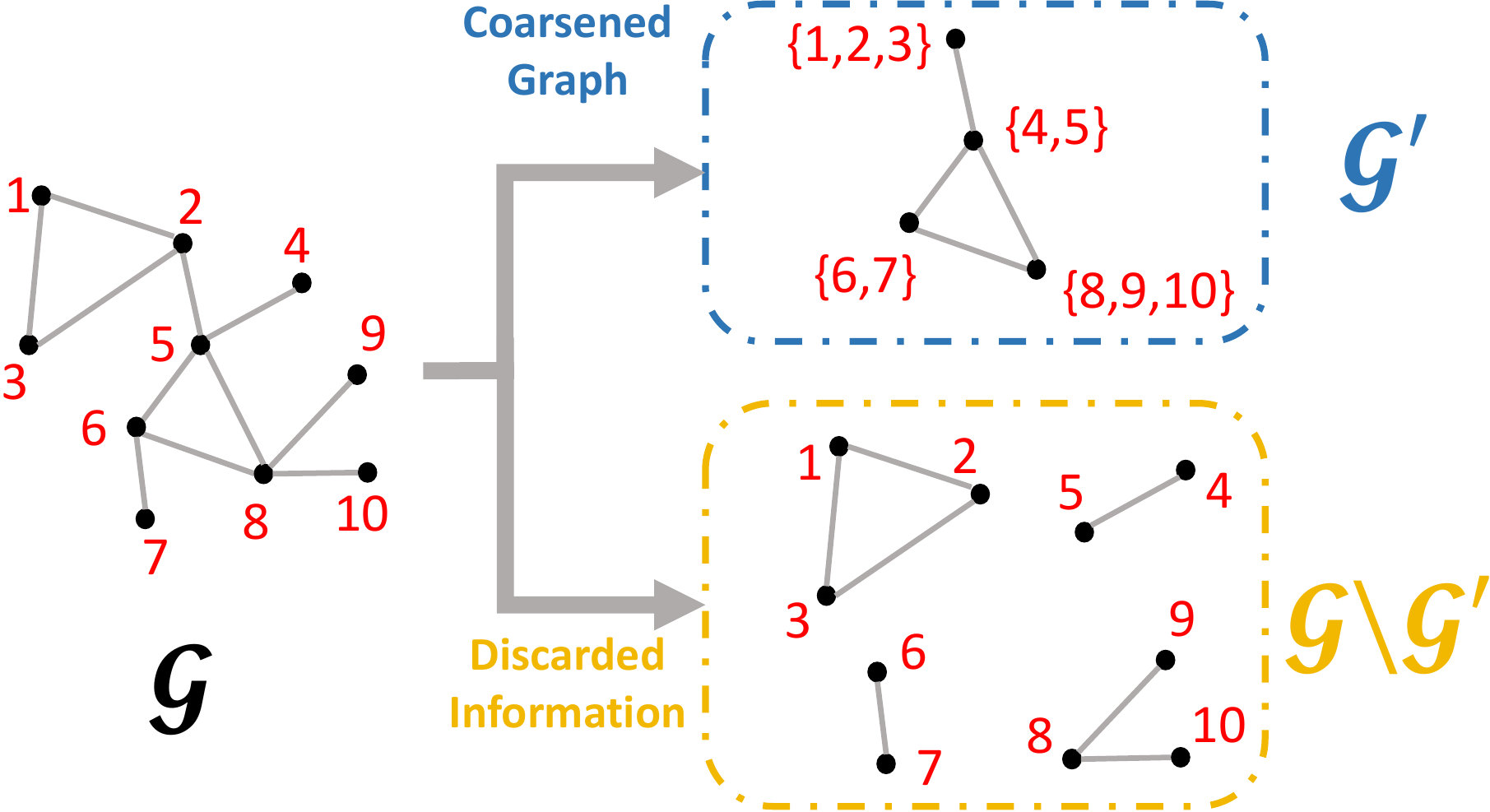}
  \caption{Illustration for graph coarsening: 
  $\mathcal{G}^{\prime}$ encodes inter-supernode relations; $\mathcal{G} \setminus \mathcal{G}^{\prime}$ holds intra-supernode information discarded during coarsening.}
  \label{fig:coarsening-residual}
\end{figure}
\begin{proposition}
\label{proposition:bound-error}
Let the GNN trained on the original graph be $f(; \Theta^{\ast \ast})$ and the one trained on the coarsened graph be $f(; \Theta^{\ast})$, where
\begin{align}
\label{eq:theta-gnn-params}
\Theta^{\ast} &\triangleq \arg\min\limits_{\Theta} \mathcal{L}(f(A, X; \Theta), Y)\ , \\
\Theta^{\ast \ast} &\triangleq \arg\min\limits_{\Theta} \mathcal{L}(f(A^{\prime}, X^{\prime}; \Theta), Y^{\prime})\ .
\end{align}
When both models are evaluated on the original graph, the mutual information between their outputs satisfies
\begin{align}
\label{eq:MI-upperbound}
I(f(A, X; \Theta^{\ast});f(A, X; \Theta^{\ast \ast})) \leq I(f(A, X; \Theta^{\ast}); Y) - \Omega\ ,
\end{align}
where $\Omega = I( g(\mathcal{G} \setminus \mathcal{G}^{\prime}) ; Y |  f(A, X; \Theta^{\ast \ast}) )$, and $g: \mathcal{G} \setminus \mathcal{G}^{\prime} \mapsto \mathcal{Y}$ is arbitrary neural network that maps $\mathcal{G} \setminus \mathcal{G}^{\prime}$ to the node label space $\mathcal{Y}$ and maximizes the conditional mutual information $I( g(\mathcal{G} \setminus \mathcal{G}^{\prime}) ; Y |  f(A, X; \Theta^{\ast \ast}) )$.
\end{proposition}

Proof is provided in Appendix~\ref{app:proofofproposition:bound-error}. 
This proposition provides an \textbf{explanatory clarification} by linking the model performance gap to the information loss induced by graph coarsening. 
In the bound of Eq.~\eqref{eq:MI-upperbound}, the gap between the two models is expressed by the mutual information $I(f(A, X; \Theta^{\ast});f(A, X; \Theta^{\ast \ast}))$, which measures how similar the outputs of the full-graph-trained model and the coarsened-graph-trained model are. 
Since the full-graph model $f(A, X; \Theta^{\ast})$ is unaffected by coarsening, the term $I(f(A, X; \Theta^{\ast}); Y)$ remains fixed. 
Therefore, the magnitude of the model gap is determined entirely by the quantity $\Omega = I( g(\mathcal{G} \setminus \mathcal{G}^{\prime}) ; Y |  f(A, X; \Theta^{\ast \ast}) )$, which refers to the graph information lost during coarsening. 
A larger $\Omega$ forces $I(f(A, X; \Theta^{\ast});f(A, X; \Theta^{\ast \ast}))$ to be smaller, producing a larger discrepancy between the full-graph and coarsened-graph models. 
In heterophilic graphs, the discarded region $\mathcal{G} \setminus \mathcal{G}^{\prime}$ contains diverse labels and informative local structure; thus making $\Omega$ large. 
Consequently, the mutual information between the two models becomes small, yielding the substantial performance drop observed empirically. 
In contrast, homophilic graphs discard much less label-relevant content, resulting in a much smaller gap.

\section{Proposed Method: Adaptive Complementary Enhancement}\label{sec:method}

We now present our proposed method, Adaptive Complementary Enhancement (ACE). 
We begin with a motivating toy experiment and then introduce the framework in detail.

\subsection{Motivating Toy Experiment: A Simple Auxiliary Loss Matters}\label{sec:method-motivation}
As demonstrated in Section~\ref{sec:preliminaries-challenge}, the heterophily challenge primarily arises from the loss of graph information during coarsening. 
A natural solution is therefore to \textbf{reintegrate the discarded information during training}. 
With this insight, we conduct a simple toy experiment: we augment the training via coarsening pipeline with an auxiliary loss that addresses the coarsening matrix $P^{\top}C$ and the original label $Y$. 
The resulting unified loss function is stated below, where $\lambda$ denotes the trade-off hyperparameter.
\begin{align}
\label{eq:simple-uniloss}
\tilde{\mathcal{L}} = & \underbrace{\mathcal{L}\left(f(A^{\prime}, X^{\prime}; \Theta), Y^{\prime}\right)}_{\text{Primary Coarsened-Graph Training Loss}} \notag\\ 
&\quad\quad\quad + \quad \lambda \cdot \underbrace{\mathcal{L}\left(P^{\top}C \cdot f(A^{\prime}, X^{\prime}; \Theta), Y\right)}_{\text{Auxiliary Loss}}\ .
\end{align}
\begin{table}[!t]
\caption{Toy experiment results. Performance gains over the vanilla coarsening-based training are highlighted in \textbf{bold}.}
\label{table-toyexp}
\centering
\setlength{\tabcolsep}{2pt}
\resizebox{\linewidth}{!}{
\begin{tabular}{ccccc}
\toprule
Model                & Method                         & Genius                  & Gamers                  & Pokec                  \\ \midrule
\multirow{6}{*}{GCN}    & SCAL                           & 67.47                   & 47.39                   & 50.44                  \\
                        & w/ $\tilde{\mathcal{L}}$  & 71.32 (\textbf{+3.85})  & 52.86 (\textbf{+5.47})  & 55.29 (\textbf{+4.85}) \\ \cmidrule(lr){2-5} 
                        & FGC                            & 70.13                   & 40.14                   & 53.81                  \\
                        & w/ $\tilde{\mathcal{L}}$   & 75.67 (\textbf{+5.54})  & 49.26 (\textbf{+9.12})  & 57.22 (\textbf{+3.41}) \\ \cmidrule(lr){2-5} 
                        & SGBGC                          & 71.22                   & 44.62                   & 55.48                  \\
                        & w/ $\tilde{\mathcal{L}}$ & 73.81 (\textbf{+2.59})  & 53.09 (\textbf{+8.47}) & 62.42 (\textbf{+6.94}) \\ \toprule
\multirow{6}{*}{GloGNN} & SCAL                           & 59.47                   & 46.28                   & 61.58                  \\
                        & w/ $\tilde{\mathcal{L}}$  & 67.41 (\textbf{+7.94}) & 53.11 (\textbf{+6.83}) & 67.74 (\textbf{+6.16}) \\ \cmidrule(lr){2-5} 
                        & FGC                            & 63.81                   & 45.92                   & 66.91                  \\
                        & w/ $\tilde{\mathcal{L}}$   & 72.33 (\textbf{+8.52})  & 51.26 (\textbf{+5.34})  & 71.44 (\textbf{+4.53}) \\ \cmidrule(lr){2-5} 
                        & SGBGC                          & 67.33                   & 47.15                   & 63.77                  \\
                        & w/ $\tilde{\mathcal{L}}$ & 73.26 (\textbf{+5.93})  & 54.31 (\textbf{+7.17})  & 69.36 (\textbf{+5.59}) \\ \toprule
\end{tabular}}
\end{table}

We follow the same experimental setup as in Section~\ref{sec:preliminaries-challenge} and compare the augmented loss \eqref{eq:simple-uniloss} against standard coarsening-based training pipelines. 
Remarkably, as shown in Table~\ref{table-toyexp}, the introduction of the simple auxiliary loss leads to substantial and consistent improvements across all coarsening-based methods and across diverse GNN backbones on heterophilic graphs. 
A more comprehensive investigation in Appendix~\ref{app:toyexp-full} reveals that the auxiliary loss benefits homophilic graphs as well. 
In summary, these positive results highlight the strong potential of the auxiliary loss that compensates for information discarded during coarsening.

\subsection{The ACE Framework}\label{sec:method-ACE}

Motivated by the observations from the toy experiments, we develop a heterophily-aware enhancement to coarsening-based GNN training by incorporating an auxiliary loss. 
This leads to our proposed framework, ACE, which consists of two core modules: (\textit{\romannumeral1}) \textbf{\textit{Refined Projector Construction}} and (\textit{\romannumeral2}) \textbf{\textit{Training With Unified Coarsening Loss}}.

\subsubsection{Step \romannumeral1: Refined Projector Construction}\label{sec:method-ACE-projector}

\paragraph{Parameterized Projector.} 
Although the simple auxiliary loss enhance training efficacy, its construction relies on the coarsening projector $P$, which is entirely dictated by the chosen graph coarsening algorithm: either classical~\citep{Gcoarse-1,Gcoarse-2,Gcoarse-3,Gcoarse-4,Gcoarse-6}, or learnable~\citep{Gcoarse-5,Gcoarse-7-FGC,Gcoarse-8,SGBGC}. 
None of existing designs specifically incorporate heterophily-related cues into the construction of the so-called projector. 
We therefore seek to empower the projector with heterophily-awareness.

To this end, we first construct the \textit{structural affinity} matrix $S$ and the \textit{feature affinity} matrix $F$, both in $\mathbb{R}^{n \times n'}$:
\begin{align}
&S = A C^{-1} P\ , \\
&F_{ij} = \exp(-\|X_i-\mu_j\|^2)\ ; \ \mu_j = \frac{1}{|\mathcal{C}_j|}\sum_{i\in\mathcal{C}_j}X_i\ .
\label{eq:information-view}
\end{align}
Following the intuition of \textit{Algebraic Multigrid} (AMG;~\citealt{AMG}), we note the term $S_{ij}$ measures how strongly node $i$ connects to supernode $j$,
and the term $F_{ij}$ is a \textit{Kernel Regression} similarity~\citep{Kernel-reg} between node $i$ and the centroid $\mu_j$. 
Together, they encode structural and feature connections while preserving coarsening structure. 
Using these two matrices, we devise a learnable, heterophily-aware projector:
\begin{align}
\label{eq:parameterized-P}
&P_{ij}(\phi) = \frac{\exp(\alpha_{ij})}{\sum_{k\in\mathcal{N}(i)}\exp(\alpha_{ik})}\ , \\ 
&\alpha_{ij} := f_{\phi}(\left[S_{ij},F_{ij}\right])\ ,
\end{align}
where $f_{\phi}$ is an MLP parametrized with $\phi$. 
Analogous to graph attention mechanism~\citep{GAT}, $P(\phi)$ learns how each fine node should associate with each supernode, balancing structural and feature signals through~$\phi$.

\paragraph{Heterophily-Aware Optimization.} 
To ensure that $P(\phi)$ correctly encodes heterophily, we study principled learning of the projector $P$. 
Conventional attributed graph coarsening specifies the projector through optimizing a combination of Dirichlet energy $\mathcal{E}(\mathcal{G})$ and feature reconstruction $\|X-PP^{\top}X\|^{2}$~\citep{Gcoarse-7-FGC,UGC}, 
thereby promoting smooth GNN predictions—an inductive bias suited for homophilic graphs. 
However, such objectives are inapplicable under heterophily, where neighboring nodes often carry contrasting information.

To overcome this issue, we draw inspiration from~\citep{graph-anisotropic-diffusion-8}, which constructs a heterophily-aware Dirichlet energy via \textit{anisotropic diffusion}~\citep{anisotropic-diffusion}, and propose the \textbf{Anisotropic Structural Regularization} term, defined as follows:
\begin{align}
\label{eq:ASR}
& \mathcal{J}_{\text{ASR}}(\phi) = \underbrace{\sum_{i,j \in E} \omega_{ij} \cdot \|(P(\phi)\bm{\mu})_{i}-(P(\phi)\bm{\mu})_{j}\|^2}_{\text{Anisotropic Smoothness}} \notag \\
&\quad\quad\quad\quad\quad + \, \beta \cdot \underbrace{\|X-P(\phi)\bm{\mu}\|_{F}^2}_{\text{Reconstruction Cost}}
\end{align}
where $\omega_{ij} := \exp(-\|X_i-X_j\|^2)$, $\bm{\mu}$ represents the stacked centroids $\{\mu_1,\mu_2,...,\mu_{n^{\prime}}\}$ (see Eq.~\eqref{eq:information-view}), and $\beta$ is the tradeoff coefficient.

The core idea behind this design is the superior ability of graph anisotropic diffusion to model heterophilic relationships in graphs, which regular isotropic diffusion adopted in most GNNs cannot capture~\citep{graph-anisotropic-diffusion-2,graph-anisotropic-diffusion-5,graph-anisotropic-diffusion-7,graph-diffusion-survey}. 
Notably, prior methods apply anisotropic diffusion in supervised GNN training, whereas we use it in an \textbf{unsupervised} manner, relying solely on the graph data. 
Despite lack of supervision, minimizing the proposed regularization term $\mathcal{J}_{\text{ASR}}(\phi)$ ensures that the projector $P(\phi^{\ast})$ is \textit{heterophily-aware}, as shown in the following proposition:
\begin{proposition}
\label{proposition:ASR-heterophily}
Let $\phi^{\ast}$ denote the minimizer of $\mathcal{J}_{\text{ASR}}(\phi)$. 
The resulting projector $P(\phi^{\ast})$ acts as a conditional spectral filter that interpolates between two distinct regimes below based on local homophily:
\begin{itemize}[leftmargin=*,parsep=1pt,itemsep=1pt,topsep=1pt]
    \item \textbf{Heterophilic Regime:} In regions where $\|X_i - X_j\| \to \infty$, $P(\phi^{\ast})$ approaches the solution of identity reconstruction, prioritizing high-frequency signal fidelity over structural smoothness.
    \item \textbf{Homophilic Regime:} In regions where $\|X_i - X_j\| \to 0$, $P(\phi^{\ast})$ approaches the solution of isotropic laplacian smoothing, enforcing local smoothness.
\end{itemize}
\end{proposition}

The proof is provided in Appendix~\ref{app:proofofproposition:ASR-heterophily}. 
This proposition provides a \textbf{local optimality characterization} of the learned projector, demonstrating that ACE naturally adapts to both homophilic and heterophilic graph structures: In heterophilic regions, where $\|X_i - X_j\| \to \infty$, the weights $\omega_{ij}$ vanish, suppressing the smoothness term and allowing the projector to prioritize high-frequency, heterophily-specific information; In homophilic regions, where $\|X_i - X_j\| \to 0$, the smoothness constraint remains active, encouraging projector to preserve locally consistent, low-frequency structure.

By minimizing Eq.~\eqref{eq:ASR} with standard optimizers such as Adam~\citep{Adamoptimizer}, we can easily learn the projector $P(\phi^{\ast})$ adaptive to heterophilic graphs.

\subsubsection{Step \romannumeral2: Training With Joint Coarsening Loss}\label{sec:method-ACE-training}
The optimized projector $P(\phi^{\ast})$ can then be used to construct the auxiliary loss term in Eq.~\eqref{eq:simple-uniloss}. 
However, the combined loss in Eq.~\eqref{eq:simple-uniloss} is inherently inflexible: it relies on a manually chosen tradeoff hyperparameter, which does not adapt to the dynamics of training and may vary across datasets or backbones. 
To overcome this inconvenience and better leverage the supervision signals available in coarsening-based GNN training, we further enhance the primitive combined loss~\eqref{eq:simple-uniloss} by employing \textit{Homoscedastic Uncertainty weighting}~\citep{homoscedastic-uncertainty}. 
This technique interprets the primary and auxiliary objectives as a \textit{multi-task learning} problem~\citep{multi-task-learning}, allowing their relative contributions to be balanced automatically via learnable task uncertainties. 
The resulting ACE objective is:
\begin{align}
\label{eq:ACE-loss-HU}
\mathcal{L}_{\text{ACE}} =& \frac{1}{2\sigma_1^2}\underbrace{\mathcal{L}\left(P(\phi^{\ast})^{\top}C \cdot f(A^{\prime}, X^{\prime}; \Theta), Y\right)}_{\text{Auxiliary Loss}} \notag \\
& + \, \frac{1}{2\sigma_2^2}\underbrace{\mathcal{L}\left(f(A^{\prime}, X^{\prime}; \Theta), Y^{\prime}\right)}_{\text{Primary Loss}} + \, \log(\sigma_1 \sigma_2)\ ,
\end{align}
where $\sigma_1$ and $\sigma_2$ are learnable uncertainty parameters that dynamically adjust weights of the two losses during training.

\subsubsection{Complexity Analysis}\label{sec:method-ACE-complexity}

We analyze the computational overhead introduced by ACE relative to vanilla training via coarsening, 
focusing on the cost of constructing the refined projector and the additional terms incurred by the ACE optimization framework. 
We follow the notations in Section~\ref{sec:preliminaries-background}.

\paragraph{Complexity of Refined Projector Construction.} 
Constructing the refined projector $P(\phi^{\ast})$ consists of a light pre-computation stage and an iterative optimization over $T$ epochs. 
The pre-computation refers to calculating $\omega_{ij}$ over all fine-level edges in $\mathcal{O}(|E|d)$ time, and forming the structure and feature affinities restricted by the coarsening assignment $\{\mathcal{C}_{1}, \mathcal{C}_{2}, ..., \mathcal{C}_{n^{\prime}}\}$ with $\mathcal{O}(n d)$ cost. 
An optimization step then includes a forward pass of the parametric MLP map $f_{\phi}$ over the sparse candidate pairs with $\mathcal{O}(n)$ cost; 
the sparse lifting operation ($P(\phi)\bm{\mu}$ are then performed with $\mathcal{O}(nd)$ cost; ultimately, obtaining the reconstruction and anisotropic smoothness terms costs $\mathcal{O}(nd)$ and $\mathcal{O}(|E|d)$ time, respectively.

Aggregating all the components, the total complexity of refined projector construction is $\boldsymbol{\mathcal{O}(T(nd + |E|d))}$, which is linear in the graph size and remains lightweight compared with a standard GNN layer.

\paragraph{Overhead in Training with Joint Coarsening Loss.} 
Training with the ACE loss is compatible with the standard coarsened-graph pipeline on $\mathcal{G}^{\prime}$, while adding a step to compute the auxiliary loss by lifting the coarsened logits back to the fine graph. 
This lifting step incurs an additional cost of approximately $\boldsymbol{\mathcal{O}(nc)}$ time. Since the feature dimension $d$ is typically much larger than the number of classes $c$, this overhead is negligible relative to the dominant cost of vanilla coarsening-based training, which is on the order of $\mathcal{O}(nd^2 + |E^{\prime}|d)$ (let $E^{\prime}$ be the edge set of $\mathcal{G}^{\prime}$).

\paragraph{Remark.} In summary, ACE improves performance by re-leveraging fine-grained supervision while maintaining the efficiency advantage of coarsening-based training. 
The inference complexity remains identical to the baseline, as the auxiliary branch is discarded after training. 
This claim is verified in the following section, where ACE is shown to improve performance with minimal computational overhead.

\section{Experiments}\label{sec:experiments}
We conduct extensive experiments to evaluate the effectiveness of ACE across diverse coarsening-based GNN training pipelines. 
Our evaluation covers both large-scale heterophilic and homophilic benchmarks and multiple GNN backbones. 
We organize the study around the following research questions:
\begin{itemize}[leftmargin=15pt,parsep=2pt,itemsep=2pt,topsep=2pt]
    \item {\bf RQ1:} How much does ACE improve performance over vanilla coarsening training pipelines?
    \item {\bf RQ2:} How does the refined projector $P(\phi^{\ast})$ influence training behavior?
    \item {\bf RQ3:} How does ACE influence the scalability and efficiency of coarsening-based training?
    \item {\bf RQ4:} How does ACE narrow the performance gap between conventional methods and state-of-the-art ones?
\end{itemize}

\subsection{Experimental Setup}\label{sec:experiments-setup}
We evaluate four state-of-the-art coarsening-based GNN training pipelines: SCAL~\citep{GC-scal}, FGC~\citep{Gcoarse-7-FGC}, UGC~\citep{UGC}, and SGBGC~\citep{SGBGC}. 
For our ACE, we parameterize $f_{\phi}$ as a two-layer MLP with 500 hidden dimensions and tune the trade-off coefficient $\beta$ from \{0.1,1,10\} during projector optimization.

We use five large-scale heterophilic datasets from~\citep{dataset6-large-hetero}: Genius, Gamers, Pokec, Snap, and arXiv-year, and two large homophilic datasets from the OGB~\citep{dataset5-ogb}: Ogbn-arXiv and Ogbn-products.
Full dataset statistics are provided in Appendix~\ref{app:expdetails-dataset}.

We consider six GNN backbones: GCN~\citep{GCN}, LINKX~\citep{dataset6-large-hetero}, GLOGNN~\citep{glognn++}, GPRGNN~\citep{GPRGNN}, SGFormer~\citep{SGFormer}, and Polynormer~\citep{Polynormer}, covering models specialized for heterophily, homophily, or both. 
\textbf{Due to space constraints}, we report results for three backbones (GCN, LINKX, and GLOGNN) in the main text, and \textbf{defer the remaining results to Appendix~\ref{app:expdetails-moreexp}}.

We set a 10\% coarsening ratio for all pipelines and finetune based on the recommended configurations from the original papers and benchmark repositories. 
All models are trained for 500 epochs using the Adam optimizer~\citep{Adamoptimizer}, with early stopping triggered after 50 epochs of no validation improvement. 
Training follows a consistent configuration across all methods: learning rate $0.01$, weight decay $5e-4$, and dropout $0.5$. 
All reported numbers are averaged over 10 independent runs.

Additional training details and complementary experimental results are provided in Appendix~\ref{app:expdetails}.


\begin{table*}[!t]
\caption{Performance of integrating ACE into GCN, LINKX, and GloGNN. Improvements on heterophilic datasets are highlighted in \textbf{bold}. Standard deviations are reported in Table~\ref{table-mainexp-1-std}; results for the remaining pipelines are provided in Table~\ref{table-mainexp-2}.}
\label{table-mainexp-1}
\centering
\setlength{\tabcolsep}{4pt}
\resizebox{\textwidth}{!}{
\begin{tabular}{ccccccc|cc}
\toprule
\multirow{2}{*}{Backbone} & \multirow{2}{*}{Method} & \multicolumn{5}{c|}{Heterophilic Graphs}                                                                                       & \multicolumn{2}{c}{Homophilic Graphs} \\ \cmidrule(lr){3-9} 
                          &                         & Genius                  & Gamers                  & Pokec                   & Snap                    & arXiv-year             & ogbn-arXiv             & ogbn-products          \\ \midrule
\multirow{9}{*}{GCN}      & -                       & 87.42                   & 62.18                   & 75.45                   & 45.65                   & 46.02                  & 71.74             & 75.64             \\ \cmidrule(lr){2-9} 
                          & SCAL                    & 67.47                   & 47.39                   & 50.44                   & 22.66                   & 27.59                  & 63.42             & 68.37             \\
                          & + ACE                   & 77.63 (\textbf{+10.16}) & 55.63 (\textbf{+8.24})  & 58.12 (\textbf{+7.68})  & 33.47 (\textbf{+10.81}) & 34.53 (\textbf{+6.94}) & 66.03 (+2.61)     & 70.38 (+2.01)     \\ \cmidrule(lr){2-9} 
                          & FGC                     & 70.13                   & 40.14                   & 53.81                   & 23.06                   & 25.49                  & 65.52             & 71.28             \\
                          & + ACE                   & 77.32 (\textbf{+7.19})  & 52.52 (\textbf{+12.38}) & 61.67 (\textbf{+7.86})  & 32.73 (\textbf{+9.67})  & 31.91 (\textbf{+6.42}) & 67.02 (+1.50)     & 72.47 (+1.19)     \\ \cmidrule(lr){2-9} 
                          & UGC                     & 69.37                   & 42.70                   & 51.74                   & 21.49                   & 29.18                  & 63.06             & 68.92             \\
                          & + ACE                   & 75.28 (\textbf{+5.91})  & 53.83 (\textbf{+11.13}) & 59.37 (\textbf{+7.63})  & 30.82 (\textbf{+9.33})  & 33.49 (\textbf{+4.31}) & 64.87 (+1.81)     & 70.19 (+1.27)     \\ \cmidrule(lr){2-9} 
                          & SGBGC                   & 71.22                   & 44.62                   & 55.48                   & 25.16                   & 27.02                  & 64.12             & 71.96             \\
                          & + ACE                   & 76.26 (\textbf{+5.04})  & 56.28 (\textbf{+11.66}) & 64.04 (\textbf{+8.56})  & 31.93 (\textbf{+6.77})  & 33.75 (\textbf{+6.73}) & 66.22 (+2.10)     & 73.13 (+1.17)     \\ \midrule
\multirow{9}{*}{LINKX}    & -                       & 90.77                   & 66.06                   & 82.04                   & 61.95                   & 56.00                  & 69.54             & 74.59             \\ \cmidrule(lr){2-9} 
                          & SCAL                    & 60.62                   & 50.13                   & 56.48                   & 40.13                   & 42.66                  & 57.71             & 62.49             \\
                          & + ACE                   & 73.29 (\textbf{+12.67}) & 57.11 (\textbf{+6.98})  & 65.77 (\textbf{+9.29})  & 47.64 (\textbf{+7.51})  & 47.18 (\textbf{+4.52}) & 61.02 (+3.31)     & 64.84 (+2.35)     \\ \cmidrule(lr){2-9} 
                          & FGC                     & 65.49                   & 42.48                   & 57.63                   & 38.17                   & 44.37                  & 59.38             & 67.52             \\
                          & + ACE                   & 74.71 (\textbf{+9.22})  & 52.33 (\textbf{+9.85})  & 70.47 (\textbf{+12.84}) & 49.08 (\textbf{+10.91}) & 47.65 (\textbf{+3.28}) & 61.73 (+2.35)     & 69.21 (+1.69)     \\ \cmidrule(lr){2-9} 
                          & UGC                     & 63.61                   & 44.55                   & 52.38                   & 41.22                   & 41.49                  & 55.46             & 64.37             \\
                          & + ACE                   & 71.08 (\textbf{+7.47})  & 50.69 (\textbf{+6.14})  & 67.24 (\textbf{+14.86}) & 47.33 (\textbf{+6.11})  & 46.21 (\textbf{+4.72}) & 60.28 (+4.82)     & 66.52 (+2.15)     \\ \cmidrule(lr){2-9} 
                          & SGBGC                   & 62.68                   & 47.60                   & 54.18                   & 42.91                   & 45.42                  & 61.55             & 64.93             \\
                          & + ACE                   & 71.53 (\textbf{+8.85})  & 56.48 (\textbf{+8.88})  & 63.24 (\textbf{+9.06})  & 46.47 (\textbf{+3.56})  & 49.02 (\textbf{+3.60}) & 63.88 (+2.33)     & 66.87 (+1.94)     \\ \midrule
\multirow{9}{*}{GloGNN}   & -                       & 90.66                   & 66.19                   & 83.00                   & 62.09                   & 54.68                  & 72.68             & 77.48             \\ \cmidrule(lr){2-9} 
                          & SCAL                    & 59.47                   & 46.28                   & 61.58                   & 43.84                   & 40.91                  & 57.09             & 66.53             \\
                          & + ACE                   & 72.29 (\textbf{+12.82}) & 55.46 (\textbf{+9.18})  & 70.34 (\textbf{+8.76})  & 49.28 (\textbf{+5.44})  & 46.26 (\textbf{+5.35}) & 59.20 (+2.11)     & 67.91 (+1.38)     \\ \cmidrule(lr){2-9} 
                          & FGC                     & 63.81                   & 45.92                   & 66.91                   & 42.70                   & 43.28                  & 61.73             & 70.46             \\
                          & + ACE                   & 74.77 (\textbf{+10.96}) & 53.42 (\textbf{+7.50})  & 74.27 (\textbf{+7.36})  & 50.38 (\textbf{+7.68})  & 46.62 (\textbf{+3.34}) & 63.78 (+2.05)     & 71.93 (+1.47)     \\ \cmidrule(lr){2-9} 
                          & UGC                     & 61.49                   & 42.96                   & 64.04                   & 45.61                   & 40.36                  & 60.29             & 68.74             \\
                          & + ACE                   & 73.68 (\textbf{+12.19}) & 51.77 (\textbf{+8.81})  & 72.47 (\textbf{+8.43})  & 49.71 (\textbf{+4.10})  & 45.24 (\textbf{+4.88}) & 63.21 (+2.92)     & 70.87 (+2.13)     \\ \cmidrule(lr){2-9} 
                          & SGBGC                   & 67.33                   & 47.15                   & 63.77                   & 45.26                   & 41.27                  & 61.24             & 69.11             \\
                          & + ACE                   & 78.43 (\textbf{+11.10}) & 56.29 (\textbf{+9.14})  & 74.55 (\textbf{+10.78}) & 52.82 (\textbf{+7.56})  & 47.50 (\textbf{+6.23}) & 63.63 (+2.39)     & 72.18 (+3.07)     \\ \bottomrule
\end{tabular}}
\end{table*}

\subsection{RQ1: Performance Gains from ACE}\label{sec:experiments-mainexp}
We examine how ACE improves the established coarsening training pipelines.
As shown in Tables~\ref{table-mainexp-1} and~\ref{table-mainexp-2}, adding ACE yields consistent and significant gains across all backbones and datasets. 
The gains are notably striking on heterophilic graphs—where coarsening tends to discard more essential graph information—\textbf{with increases of up to 14.86}. 
These confirm that ACE markedly enhances the efficacy of coarsening training, especially under heterophilic conditions.

Beyond these overall improvements, ACE also significantly outperforms the naive auxiliary loss used in the toy experiments. 
By using the refined projector, ACE achieves \textbf{roughly twice the performance gain} of the naive formulation. 
This highlights the value of our refinement: it not only strengthens the effect of reintegrating the discarded information but also validates the central premise that restoring lost information is essential for handling heterophily in coarsening-based training.

Moreover, ACE also benefits other graph-reduction methods beyond coarsening, such as \textit{graph condensation}~\citep{Gcond} (Table~\ref{table-convmatch-gcond}), and remains robust even under \textbf{extreme coarsening ratio} (Table~\ref{table-0.01coarse}), highlighting its generality and broad applicability.

\subsection{RQ2: Impact of the Refined Projector}\label{sec:experiments-projector}
The optimization of Eq.~\eqref{eq:ASR} for obtaining $P(\phi^{\ast})$ is governed by a trade-off hyperparameter $\beta$, which balances the reconstruction and anisotropic regularization terms and thus shapes the refined projector. 
To evaluate how sensitive the learned projector—and consequently the coarsening training—is to this choice, we conduct an ablation study over a wide range of $\beta$ within \{0.01,0.1,1,10,100\}. 
We perform this analysis using two strong pipelines (SCAL and SGBGC) and two large heterophilic benchmarks (Pokec and Snap). 
Additional results for UGC and FGC are provided in Appendix~\ref{app:expdetails-moreexp}.

\begin{figure*}[!t]
  \centering
    \subfloat[SCAL on Pokec.]{
    \label{fig:beta-ablation-scal-pokec}
    \includegraphics[width=0.235\linewidth]{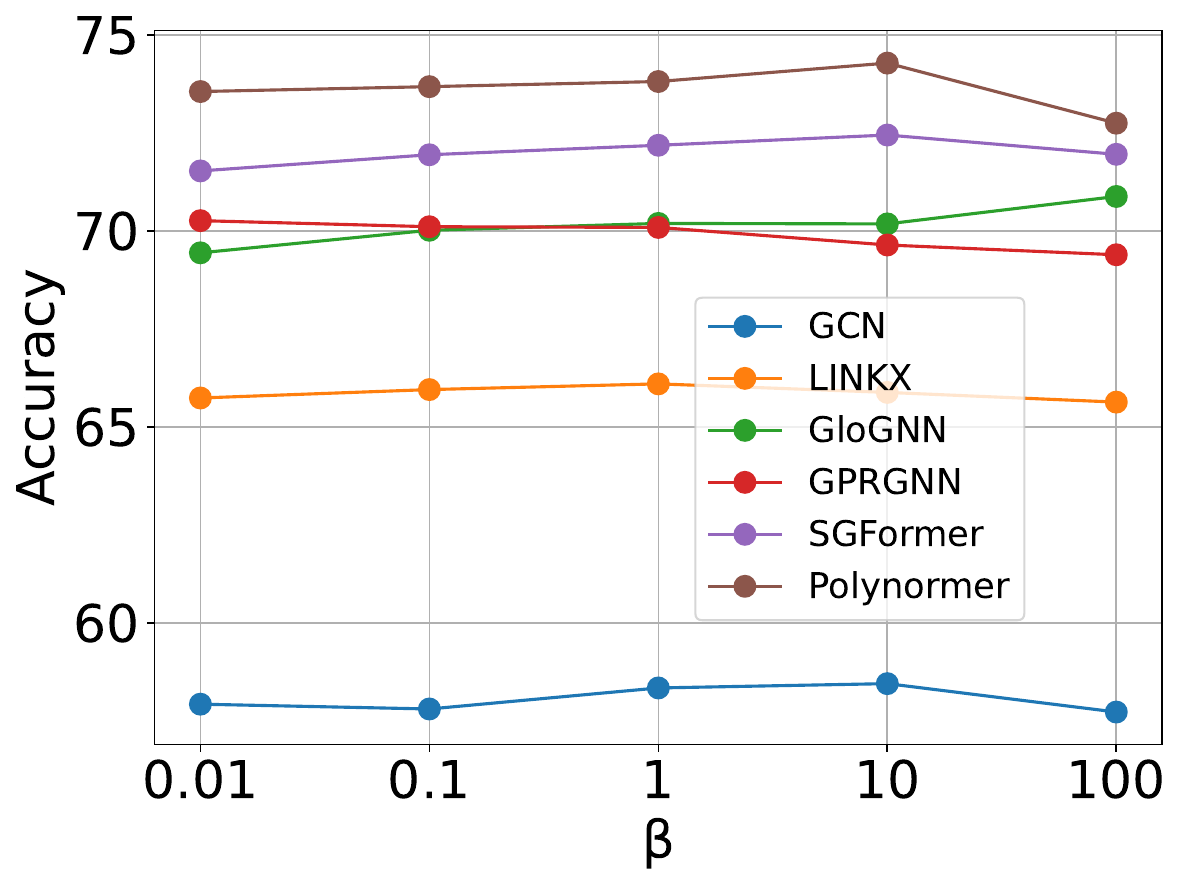}
    }
    \subfloat[SCAL on Snap.]{
    \label{fig:beta-ablation-scal-snap}
    \includegraphics[width=0.235\linewidth]{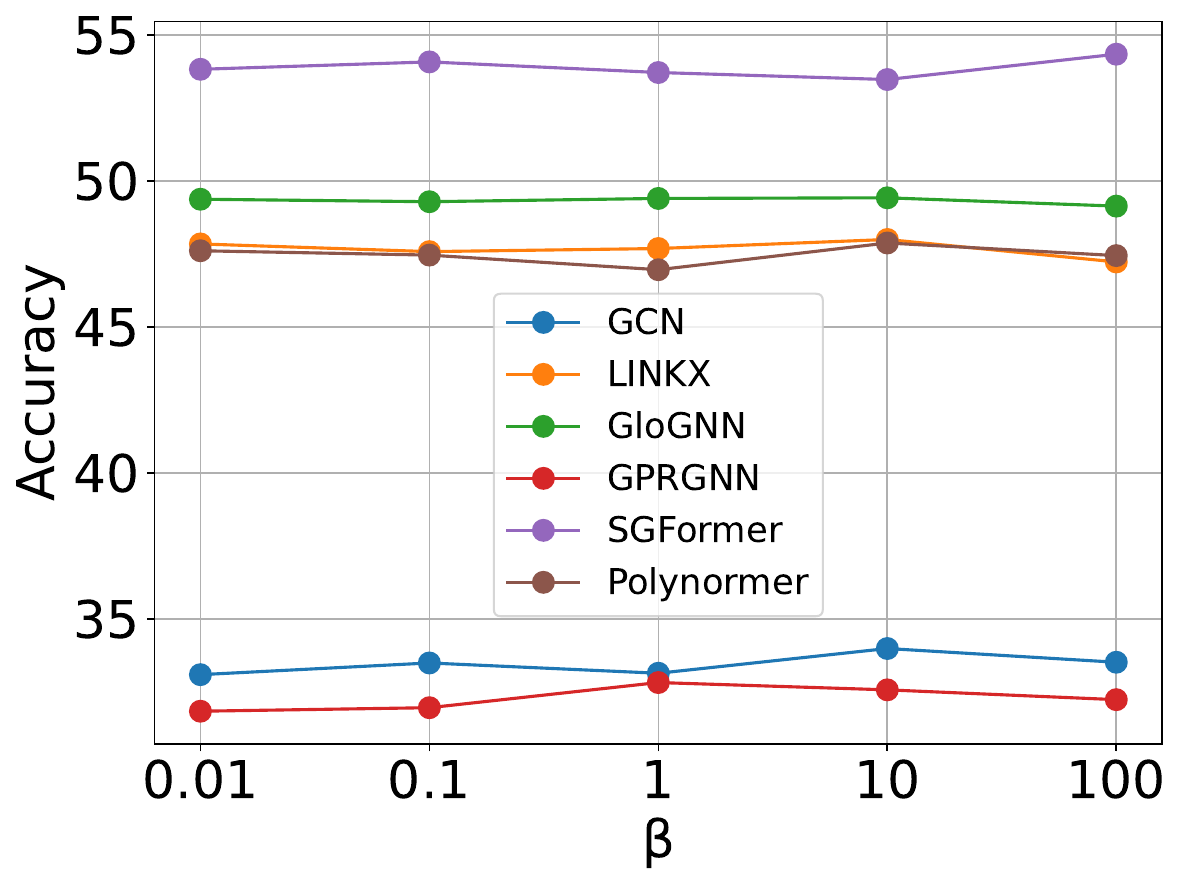}
    }
    \subfloat[SGBGC on Pokec.]{
    \label{fig:beta-ablation-sgbgc-pokec}
    \includegraphics[width=0.235\linewidth]{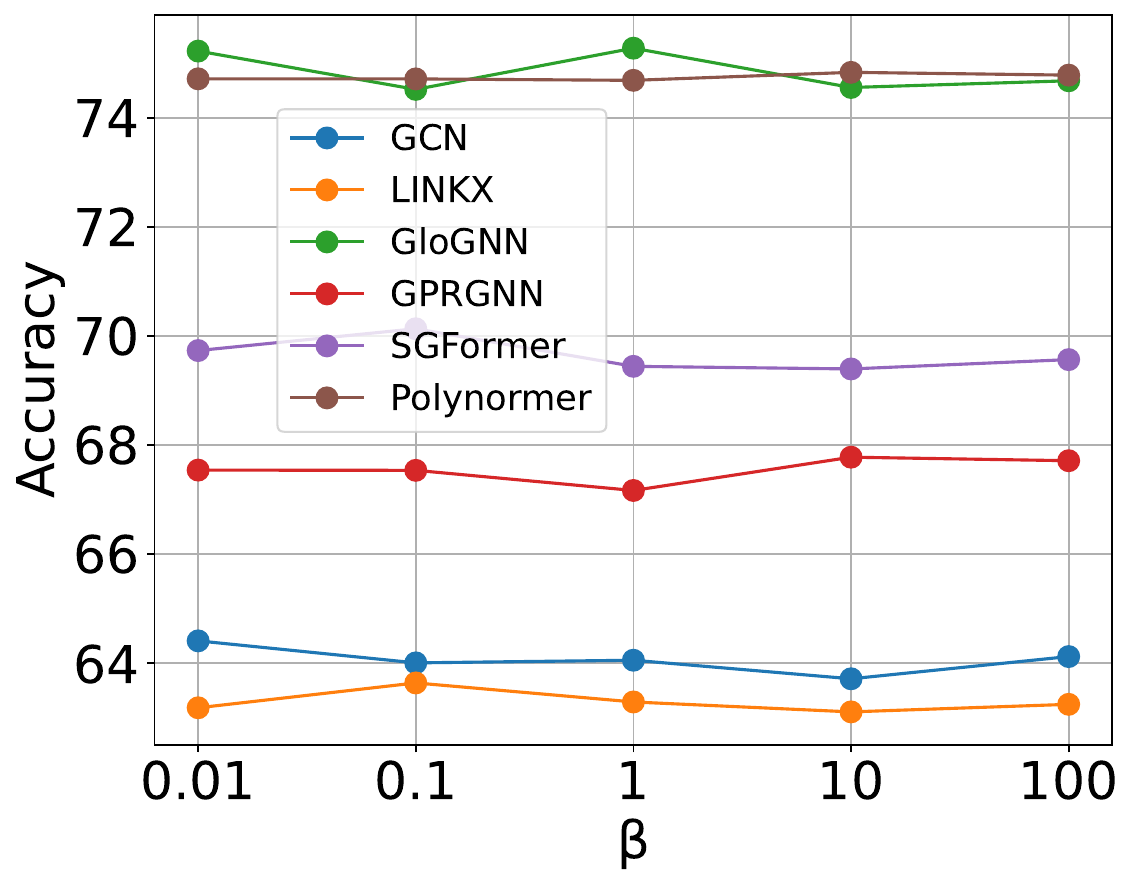}
    }
    \subfloat[SGBGC on Snap.]{
    \label{fig:beta-ablation-sgbgc-snap}
    \includegraphics[width=0.235\linewidth]{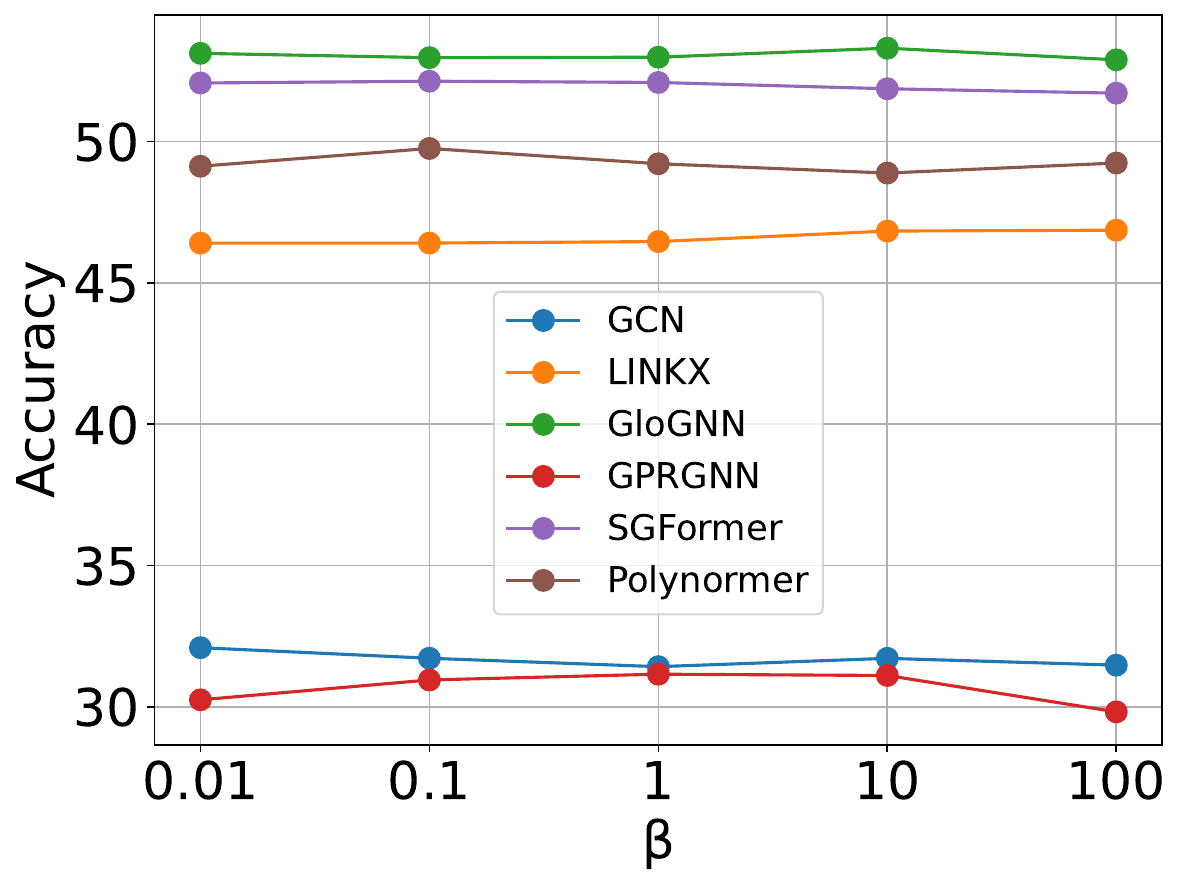}
    }
  \caption{Effect of $\beta$ across coarsening pipelines, GNN backbones, and datasets. \textbf{More results can be found in Figure~\ref{fig:beta-ablation-full}.}}
  \label{fig:beta-ablation}
\end{figure*}

Figure~\ref{fig:beta-ablation} (and Figure~\ref{fig:beta-ablation-full}) showcase that varying $\beta$ across two orders of magnitude produces only minor performance differences, consistently across datasets and backbone GNNs. 
This highlights that the learning dynamics of the refined projector $P(\phi^{\ast})$ are highly robust to the choice of $\beta$. 
Although occasional small fluctuations can be observed, the overall trend remains stable, and setting $\beta$ be 1 already yields superior performance without further tuning. 
To sum up, such insensitivity to hyperparameter choice makes ACE \textbf{practically convenient}, requiring minimal effort to deploy while preserving strong performance across diverse settings.

\paragraph{Beyond the Projection Module.} Readers are encouraged to consult the additional ablation studies in Appendix~\ref{app:expdetails-moreexp} (Tables~\ref{table-ablation-full-GCN} and~\ref{table-ablation-full-GPRGNN}), which investigate the remaining components of ACE and offer an extensive understanding of their isolated contributions beyond the projector module.


\begin{figure}[!t]
  \centering
    \subfloat[Preprocessing time: SCAL vs. SCAL+ACE]{
    \label{fig:preprocessing-scal}
    \includegraphics[width=.99\linewidth]{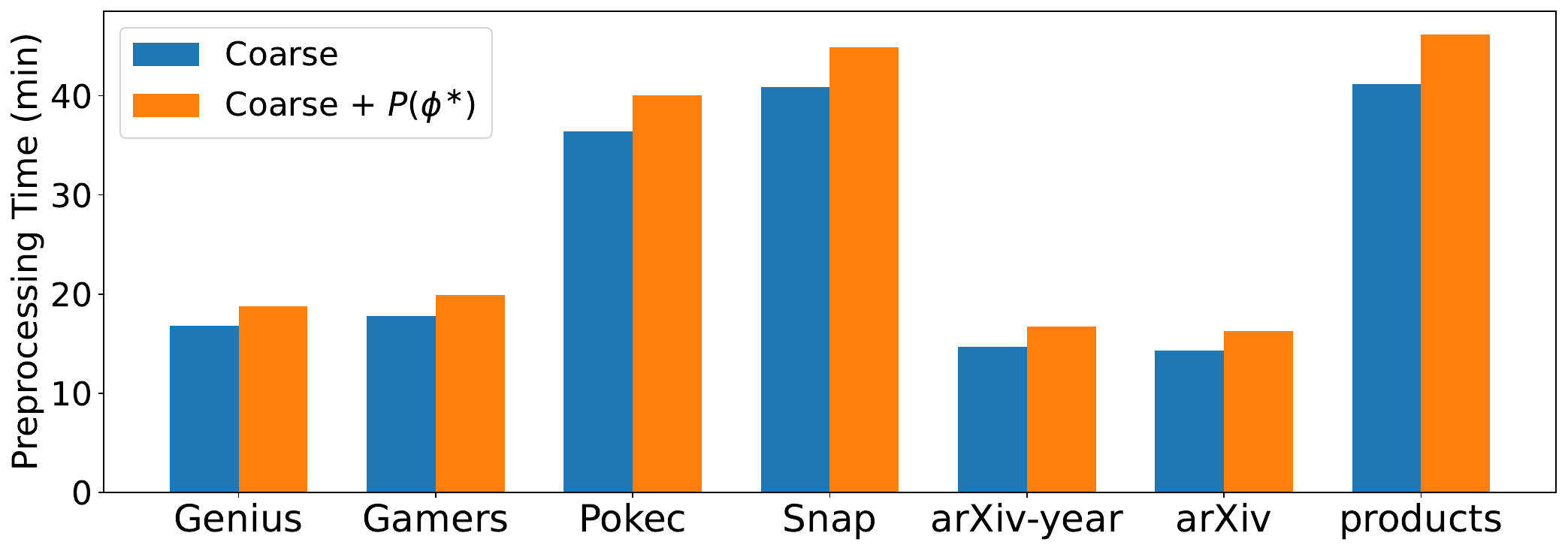}
    } \\
    \subfloat[Preprocessing time: SGBGC vs. SGBGC+ACE]{
    \label{fig:preprocessing-sgbgc}
    \includegraphics[width=.99\linewidth]{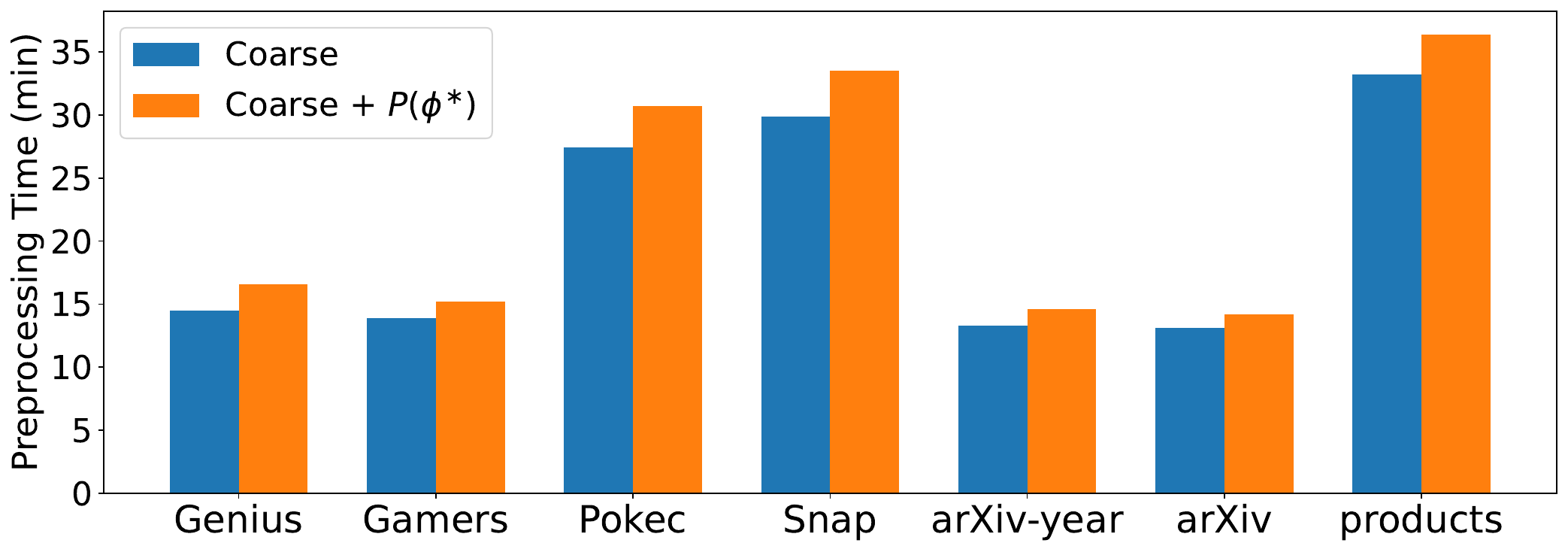}
    }
    \caption{Comparison of preprocessing time for SCAL, SGBGC, and their ACE-enhanced variants. \textbf{Complete results are provided in Figure~\ref{fig:preprocessing-full}.}}
    \label{fig:preprocessing}
\end{figure}

\subsection{RQ3: Effects of ACE on Efficiency and Resource Consumption}\label{sec:experiments-scalability}

\paragraph{Preprocessing Overhead.} Coarsening training pipelines require a preprocessing stage to construct coarsened graphs, typically incurring substantial cost (often on the order of $n\cdot E$). 
ACE adds one additional step—learning the refined projector $P(\phi^{\ast})$—but, as discussed in Section~\ref{sec:method-ACE-complexity}, this operation scales only with $n + E$ level and is therefore lightweight relative to the core coarsening procedure.

Figure~\ref{fig:preprocessing} (with full results in Figure~\ref{fig:preprocessing-full}) confirms this analysis: ACE introduces only a negligible increase in preprocessing time across all pipelines and datasets. 
This validates that the refined projector construction is highly efficient, allowing ACE to improve coarsened representations without altering the overall preprocessing cost profile.

\paragraph{Training Overhead.} We next assess the runtime and memory efficiency of ACE during model training, using the large-scale Pokec dataset and all six backbone GNNs. 
As shown in Figure~\ref{fig:time-space} (with additional results in Figures~\ref{fig:training-time-full} and \ref{fig:gpu-memory-full}), ACE adds only a small overhead: approximately \textbf{4\%–7\%} in both training time per epoch and GPU memory usage. 
Given that ACE delivers over 10\% accuracy improvement on challenging heterophilic graphs, this modest overhead represents a highly favorable trade-off and highlights the practicality of integrating ACE into existing coarsening-based pipelines.

\begin{figure}[!t]
  \centering
    \subfloat[SCAL training time per epoch on Pokec.]{
    \label{fig:training-time}
    \includegraphics[width=.99\linewidth]{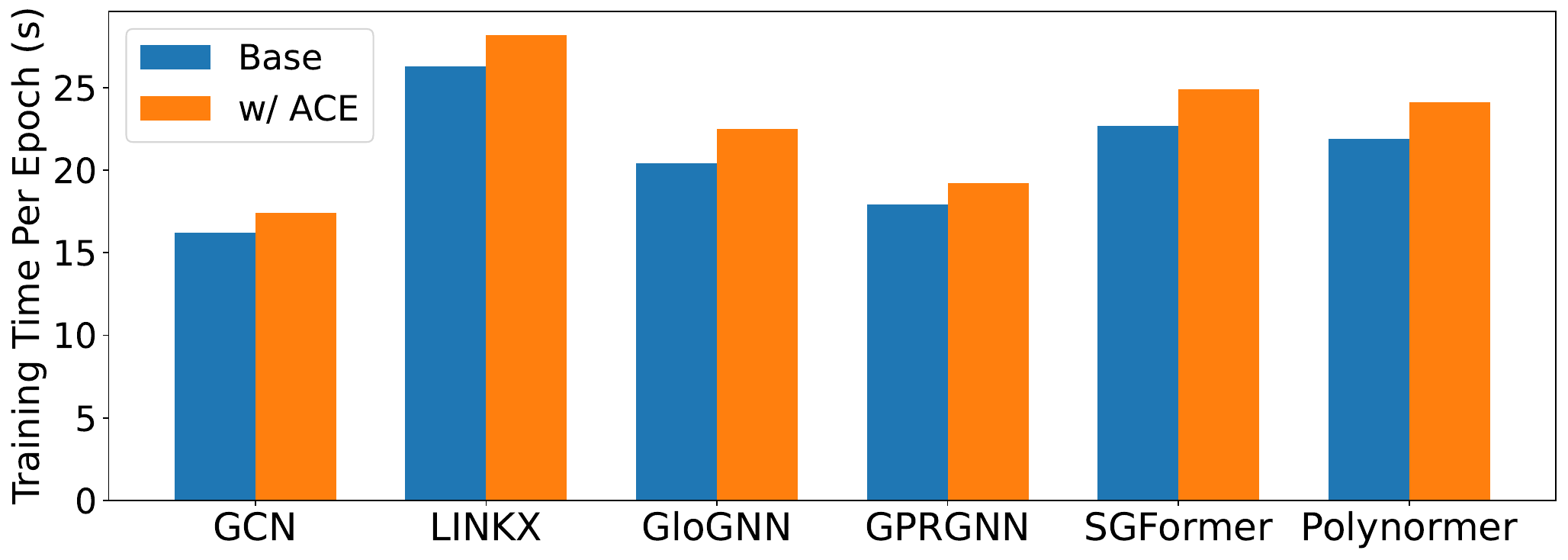}
    }\\
    \subfloat[SCAL training memory usage on Pokec.]{
    \label{fig:gpu-memory}
    \includegraphics[width=.99\linewidth]{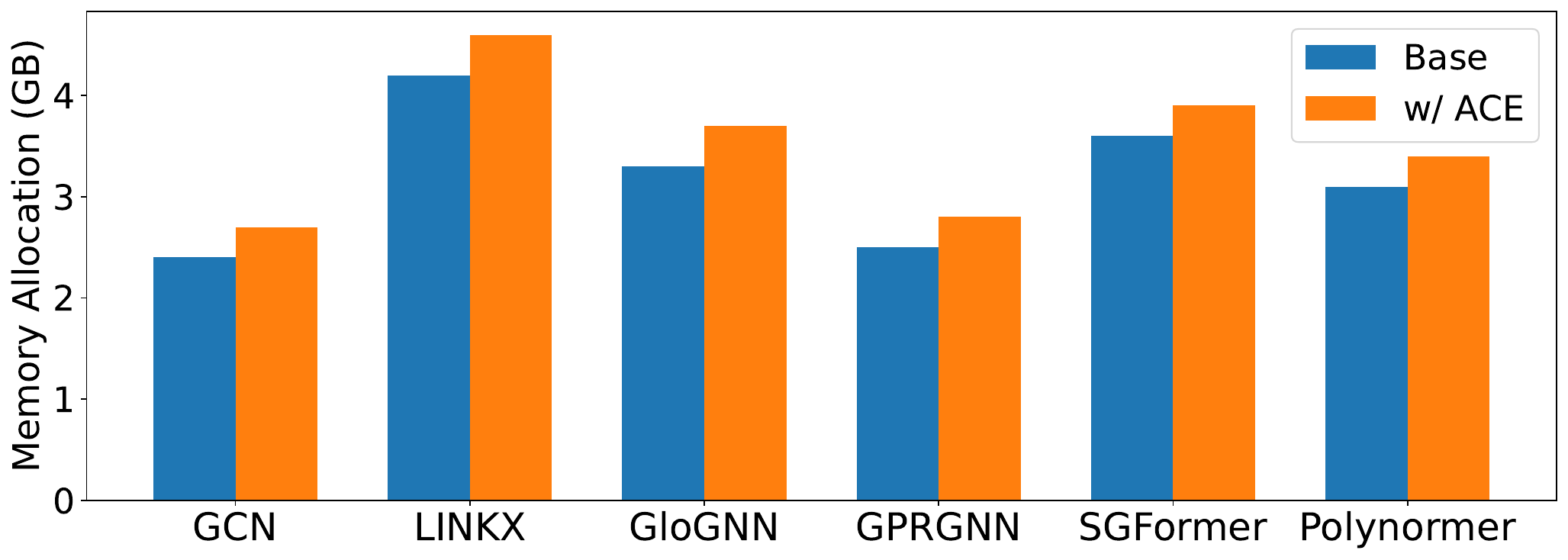}
    }
    \caption{Training memory usage and per-epoch runtime (on Pokec) of SCAL and SCAL+ACE across diverse backbone GNNs. \textbf{More results are provided in Tables~\ref{fig:training-time-full} and~\ref{fig:gpu-memory-full}.}}
    \label{fig:time-space}
\end{figure}


\begin{table*}[!th]
\centering
\caption{Comparison with state-of-the-art scalable GNN training methods based on graph data reduction. Results on a \textbf{GCN} backbone with a 10\% reduction ratio are reported in terms of test accuracy and total training time. \textbf{Complementary results are deferred to Table~\ref{table-againstSOTA-GPRGNN}.}}
\resizebox{\linewidth}{!}{
\begin{tabular}{cccc|cc|cc|cc}
\toprule
\multirow{2}{*}{Category}      & \multirow{2}{*}{Method} & \multicolumn{2}{c|}{Pokec} & \multicolumn{2}{c|}{Snap} & \multicolumn{2}{c|}{Gamers} & \multicolumn{2}{c}{Ogbn-product} \\ \cmidrule(lr){3-10} 
                               &                         & ACC       & Time (min)     & ACC       & Time (min)    & ACC        & Time (min)     & ACC          & Time (min)        \\ \midrule
\multirow{2}{*}{Coarsening}    & FGC                     & 53.81     & 24.73          & 23.06     & 31.66         & 40.14      & 12.39          & 71.28        & 32.71             \\
                               & UGC                     & 51.74     & 23.69          & 21.49     & 29.19         & 42.70      & 11.23          & 68.92        & 31.41             \\ \midrule
Data-Adaptive                  & GECC                    & 60.66     & 33.82          & 29.31     & 38.77         & 50.66      & 18.43          & 71.69        & 45.29             \\ \midrule
Subgraph-Sampling                        & AGS                     & 57.32     & 21.81          & 28.32     & 29.48         & 48.81      & 10.13          & 71.02        & 30.18             \\ \midrule
\multirow{2}{*}{\textbf{Ours}} & FGC+ACE                 & \textbf{61.67}     & 25.67          & \textbf{32.73}     & 32.92         & \textbf{52.52}      & 12.89          & \textbf{72.47}        & 33.65             \\
                               & UGC+ACE                 & 59.37     & 24.34          & \textbf{30.82}     & 31.25         & \textbf{53.83}      & 11.72          & 70.19        & 32.24             \\ \bottomrule
\end{tabular}}
\label{table-againstSOTA-GCN}
\end{table*}

\subsection{RQ4: Competitiveness of Conventional Methods with ACE}\label{sec:experiments-againstSOTA}

To assess the position of ACE among existing scalable and efficient GNN training paradigms, we compare ACE-enhanced conventional coarsening methods—specifically FGC and UGC—with representative state-of-the-art approaches beyond coarsening, including the sampling-based method AGS~\citep{AGS} and the graph condensation method GECC~\cite{GECC}. 
Results using the GCN backbone are presented here, while those with GPRGNN backbone are deferred to Appendix~\ref{app:expdetails-moreexp}.

As shown in Table~\ref{table-againstSOTA-GCN} (with additional results in Table~\ref{table-againstSOTA-GPRGNN}), the conventional coarsening methods FGC and UGC substantially underperform state-of-the-art approaches when used alone. 
However, after incorporating ACE, both methods exhibit significant performance gains, achieving results that are comparable to, and in several cases surpass, those of the state-of-the-art baselines. 
These findings suggest that ACE can substantially enhance the competitiveness of conventional coarsening-based training methods. 
Furthermore, ACE introduces only negligible computational overhead relative to the compared baselines, resulting in a more favorable performance-cost trade-off. 
Overall, the results highlight ACE as a practical and promising direction for revitalizing conventional coarsening-based training methods in the era of scalable graph learning.

\section{Conclusion}\label{sec:conclusion}

\paragraph{Paper Summary.} We study how graph heterophily affects GNN performance in coarsening-based GNN training. 
Guided by empirical observations and supported by our theoretical analysis, we interpret the degradation as that for heterophilic graphs coarsening discards more essential graph information than for homophilic graphs, which is an underexplored perspective in prior research. 
To address this issue, we introduced ACE, which first learns a heterophily-aware refined projector via Anisotropic Structural Regularization and incorporates it into an auxiliary loss; 
we then combine the designs and propose an improved optimization objective via Homoscedastic Uncertainty weighting, thus adaptively reintegrating discarded information to enhance coarsening-based training. 
Across diverse coarsening training pipelines, benchmarks, and backbone GNNs, ACE consistently achieves superior performance gains with minimal computational overhead, offering an efficient and easy-to-deploy enhancement for practical coarsening training.

\paragraph{Limitations and Future Work.} Although ACE substantially improves coarsening-based training, a non-negligible performance gap remains compared to full-graph training. 
This is primarily because ACE recovers discarded information implicitly via the refined projector, rather than modeling explicit intra-supernode structure. 
While this design preserves scalability and efficiency, it may limit the achievable performance. 
A promising direction for future work is to develop more expressive mechanisms for integrating discarded information, extending the ACE framework to further close this gap without sacrificing scalability.

\begin{acknowledgements} 

Yifan Chen acknowledges funding from 
National Natural Science Foundation of China (NSFC) under grant \texttt{62502408},
Research Grants Council (RGC) under grant \texttt{22303424}, GuangDong Basic and Applied Basic Research Foundation under grant \texttt{2025A1515010259},
and Guangdong and Hong Kong Universities ``1+1+1'' Joint Research Collaboration Scheme under grant \texttt{2025A050500}.
\end{acknowledgements}

\bibliography{uai_submission}


\clearpage
\appendix
\onecolumn

\section{Proof of Proposition~\ref{proposition:bound-error}}\label{app:proofofproposition:bound-error}

{\bf Proposition~\ref{proposition:bound-error}.} \textit{
Let the GNN trained on the original graph be $f(; \Theta^{\ast \ast})$ and the one trained on the coarsened graph be $f(; \Theta^{\ast})$, where
\begin{align}
\Theta^{\ast} &\triangleq \arg\min\limits_{\Theta} \mathcal{L}(f(A, X; \Theta), Y)\ , \\
\Theta^{\ast \ast} &\triangleq \arg\min\limits_{\Theta} \mathcal{L}(f(A^{\prime}, X^{\prime}; \Theta), Y^{\prime})\ .
\end{align}
When both models are evaluated on the original graph, the mutual information between their outputs satisfies
\begin{align}
I(f(A, X; \Theta^{\ast});f(A, X; \Theta^{\ast \ast})) \leq I(f(A, X; \Theta^{\ast}); Y) - \Omega\ ,
\end{align}
where $\Omega = I( g(\mathcal{G} \setminus \mathcal{G}^{\prime}) ; Y |  f(A, X; \Theta^{\ast \ast}) )$, and $g: \mathcal{G} \setminus \mathcal{G}^{\prime} \mapsto \mathcal{Y}$ is arbitrary neural network that maps $\mathcal{G} \setminus \mathcal{G}^{\prime}$ to the node label space $\mathcal{Y}$ and maximizes the conditional mutual information $I( g(\mathcal{G} \setminus \mathcal{G}^{\prime}) ; Y |  f(A, X; \Theta^{\ast \ast}) )$.
}

\begin{proof}
\label{proof:proposition:bound-error}
We first introduce two lemmas that will serve as the foundation for our subsequent analysis.

\begin{lemma}[(Theorem~1 ~\citep{Mutualinformation=Crossentropy-1} Restated)]
\label{lemma:Mutualinformation=Crossentropy}
In standard neural network training, the infimum of the expected cross-entropy loss with a softmax output is equivalent to the mutual information between input and output variables, up to an additive $\log c$ term (where $c$ is the number of classes), provided that all class probabilities are bounded away from zero~\citep{Mutualinformation=Crossentropy-1}.
\end{lemma}

\begin{lemma}
\label{lemma:MI-minimum}
Let $Y$ be the ground-truth label, and let $A$ and $B$ be the outputs of two
different models trained to predict $Y$. Assume that $A$ and $B$ are 
conditionally independent given $Y$, i.e.\ $A \perp B \mid Y$. Then
\[
I(A;B) \;\le\; \min\{\, I(A;Y),\; I(B;Y) \,\}.
\]
\end{lemma}
\begin{proof}
\label{proof:lemma:MI-minimum}
The conditional independence $A \perp B \mid Y$ implies the Markov chains
\[
A \to Y \to B 
\qquad\text{and}\qquad
B \to Y \to A .
\]
By the data-processing inequality applied to the first Markov chain, any
information shared between $A$ and $B$ must flow through $Y$, which yields
\[
I(A;B) \le I(A;Y).
\]
Applying data processing to the second chain analogously gives
\[
I(A;B) \le I(B;Y).
\]
Taking the minimum of the two bounds establishes
\[
I(A;B) \le \min\{ I(A;Y),\; I(B;Y) \}.
\]
Intuitively, since $A$ and $B$ can only correlate through the common target $Y$,
their mutual information cannot exceed the amount of information that each 
individually retains about $Y$. The weaker predictor therefore determines the
maximum possible dependence between them.
\end{proof}

Now we begin the proof. 
By Lemma~\ref{lemma:Mutualinformation=Crossentropy}, the optimization objectives in Eq.~\eqref{eq:theta-gnn-params} can be equivalently expressed in terms of mutual-information maximization. 
In particular,
\begin{align}
\label{eq:theta-gnn-params-eq}
\Theta^{\ast} &\triangleq \arg\max\limits_{\Theta} I(f(A, X; \Theta); Y)\ , \\
\Theta^{\ast \ast} &\triangleq \arg\max\limits_{\Theta} I(f(A^{\prime}, X^{\prime}; \Theta); Y^{\prime})\ ,
\end{align}
thus recasting the original loss-minimization problems into the perspective of mutual-information maximization.

Next, recall the chain-rule of mutual information~\citep{elementsofIT}:
\begin{equation}
\label{eq:proof:proposition:bound-error-1}
I(U , V ; W) - I(U ; W) = I(V ; W | U)
\end{equation}
Let $U = f(A, X ; \Theta^{\ast \ast})$, $V = g(\mathcal{G} \setminus \mathcal{G}^{\prime})$, $W = Y$, we obtain:
\begin{align}
\label{eq:proof:proposition:bound-error-2}
&\; I \Bigl( f(A, X ; \Theta^{\ast \ast}) , g(\mathcal{G} \setminus \mathcal{G}^{\prime}) ; Y \Bigr) - I \Bigl( f(A, X ; \Theta^{\ast \ast}) ; Y \Bigr) \notag \\ 
= &\; I\Bigl(g(\mathcal{G} \setminus \mathcal{G}^{\prime}); Y \big| f(A, X ; \Theta^{\ast \ast}) \Bigr)\ .
\end{align}
We focus on the term $I \Bigl( f(A, X ; \Theta^{\ast \ast}) , g(\mathcal{G} \setminus \mathcal{G}^{\prime}) ; Y \Bigr)$, which quantifies the relationship between two models trained on different partitions of the graph $\mathcal{G}$—namely, $\mathcal{G}^{\prime}$ and $\mathcal{G} \setminus \mathcal{G}^{\prime}$—and the labels of the entire graph, $Y$. 
Specifically, we leverage the data processing inequality~\citep{elementsofIT}:
\begin{equation}
\label{eq:proof:proposition:bound-error-3}
I(U ; V) \geq I(U ; W),\, \text{with Markov chain:}\,\, U \rightarrow V \rightarrow W\ .
\end{equation}
Notice that $f(; \Theta^{\ast})$, $f(; \Theta^{\ast \ast})$, and $g$ are all deterministic functions. 
Hence, by letting $U = Y$, $V = f(A, X; \Theta^{\ast})$, and $W = \Bigl( f(A, X ; \Theta^{\ast \ast}) , g(\mathcal{G} \setminus \mathcal{G}^{\prime}) \Bigr)$, we can derive the following inequality:
\begin{equation}
\label{eq:proof:proposition:bound-error-4}
I \Bigl( f(A, X; \Theta^{\ast}) ; Y \Bigr) \geq I \Bigl( f(A, X ; \Theta^{\ast \ast}) , g(\mathcal{G} \setminus \mathcal{G}^{\prime}) ; Y \Bigr)\ .
\end{equation}
The inequality encapsulates the intuitive principle that the aggregate contribution from the model trained on the coarsened graph, $f(A, X ; \Theta^{\ast \ast})$, and the additional term representing the information excluded in the coarsened graph, $g(\mathcal{G} \setminus \mathcal{G}^{\prime})$, is bounded above by the output of the model trained directly on the full graph, $f(A, X; \Theta^{\ast})$.

By combining Equation~\ref{eq:proof:proposition:bound-error-2} and Equation~\ref{eq:proof:proposition:bound-error-4}, and observing that $I \Bigl( f(A, X; \Theta^{\ast}); Y \Bigr) \geq I \Bigl( f(A, X; \Theta^{\ast \ast}); Y \Bigr)$, the following inequality is derived:
\begin{align}
\label{eq:proof:proposition:bound-error-5}
\Delta & = \Bigl| I\Bigl( f(A, X; \Theta^{\ast}); Y \Bigr) - I \Bigl( f(A, X; \Theta^{\ast \ast}); Y \Bigr) \Bigr| \notag \\
& \geq I \Bigl( g(\mathcal{G} \setminus \mathcal{G}^{\prime}) ; Y |  f(A, X; \Theta^{\ast \ast}) \Bigr)\ .
\end{align}

Finally, applying Lemma~\ref{lemma:MI-minimum} with $A := f(A, X; \Theta^{\ast})$ and $B := f(A, X; \Theta^{\ast \ast})$, and using the fact that
\begin{align}
I\Bigl( f(A, X; \Theta^{\ast}); Y \Bigr) \geq I \Bigl( f(A, X; \Theta^{\ast \ast}); Y \Bigr)\ ,
\end{align}
we obtain
\begin{align}
I(f(A, X; \Theta^{\ast});f(A, X; \Theta^{\ast \ast})) 
&\leq I \Bigl( f(A, X; \Theta^{\ast \ast}); Y \Bigr) \\
&\leq I\Bigl( f(A, X; \Theta^{\ast}); Y \Bigr) - I \Bigl( g(\mathcal{G} \setminus \mathcal{G}^{\prime}) ; Y |  f(A, X; \Theta^{\ast \ast}) \Bigr) \\
&= I\Bigl( f(A, X; \Theta^{\ast}); Y \Bigr) - \Omega\ ,
\end{align}
which establishes the desired inequality and completes the proof.
\end{proof}


\section{Proof of Proposition~\ref{proposition:ASR-heterophily}}\label{app:proofofproposition:ASR-heterophily}

{\bf Proposition~\ref{proposition:ASR-heterophily}.} \textit{
Let $P(\phi^{\ast})$ denote the projector obtained by solving $\min_{\phi}\mathcal{J}_{\text{ASR}}(\phi)$. The resulting $P(\phi^{\ast})$ acts as a conditional spectral filter that interpolates between two distinct regimes based on local homophily:
\begin{itemize}[leftmargin=*,parsep=1pt,itemsep=1pt,topsep=1pt]
    \item \textbf{Heterophilic Regime:} In regions where $\|X_i - X_j\| \to \infty$, $P(\phi^{\ast})$ approaches the solution of identity reconstruction, prioritizing high-frequency signal fidelity over structural smoothness.
    \item \textbf{Homophilic Regime:} In regions where $\|X_i - X_j\| \to 0$, $P(\phi^{\ast})$ approaches the solution of isotropic laplacian smoothing, enforcing local smoothness.
\end{itemize}
}

\begin{proof}
\label{proof:proposition:ASR-heterophily}
We analyze the local optimality condition for a specific node $i$ by isolating the terms in $\mathcal{L}(P)$ dependent on $(P(\phi)\bm{\mu})_i$. 
The local objective is given by:
\begin{equation}
    \mathcal{L}_i((P(\phi)\bm{\mu})_i) = \|(P(\phi)\bm{\mu})_i - X_i\|^2 + \lambda \sum_{j \in \mathcal{N}(i)} \exp(-\|X_i-X_j\|^2) \cdot \|(P(\phi)\bm{\mu})_i - (P(\phi)\bm{\mu})_j\|^2.
\end{equation}
The necessary condition for optimality requires the gradient $\nabla_{(P(\phi)\bm{\mu})_i} \mathcal{L}_i$ to vanish. Differentiating with respect to $(P(\phi)\bm{\mu})_i$ yields:
\begin{equation}
    2((P(\phi)\bm{\mu})_i - X_i) + 2\lambda \sum_{j \in \mathcal{N}(i)} \exp(-\|X_i-X_j\|^2) ((P(\phi)\bm{\mu})_i - (P(\phi)\bm{\mu})_j) = 0.
\end{equation}
Rearranging terms to isolate $(P(\phi)\bm{\mu})_i$, we obtain the closed-form update rule for the optimal lifted signal:
\begin{equation} \label{eq:update_rule}
    (P(\phi^{\ast})\bm{\mu})_i \left( 1 + \lambda \sum_{j \in \mathcal{N}(i)} \exp(-\|X_i-X_j\|^2) \right) = X_i + \lambda \sum_{j \in \mathcal{N}(i)} \exp(-\|X_i-X_j\|^2) (P(\phi)\bm{\mu})_j.
\end{equation}
This leads to the expression:
\begin{equation}
    (P(\phi^{\ast})\bm{\mu})_i = \frac{X_i + \lambda \sum_{j \in \mathcal{N}(i)} \exp(-\|X_i-X_j\|^2) (P(\phi)\bm{\mu})_j}{1 + \lambda \sum_{j \in \mathcal{N}(i)} \exp(-\|X_i-X_j\|^2)}.
\end{equation}
We now examine the asymptotic behavior of the weight coefficient $\exp(-\|X_i-X_j\|^2) \in (0, 1]$.

\textbf{Case 1 (Homophily):} Consider a neighborhood where features are locally smooth, such that $X_i \approx X_j$. Here, $\|X_i - X_j\| \to 0$, implying $\exp(-\|X_i-X_j\|^2) \to 1$. The update rule in Eq. \eqref{eq:update_rule} converges to:
\begin{equation}
    (P(\phi^{\ast})\bm{\mu})_i \approx \frac{X_i + \lambda \sum_{j \in \mathcal{N}(i)} (P(\phi)\bm{\mu})_j}{1 + \lambda d_i},
\end{equation}
where $d_i$ is the degree of node $i$. This formulation is equivalent to the first-order approximation of the graph diffusion process (Low-Pass Filter), where the node representation is smoothed by averaging over its neighbors.

\textbf{Case 2 (Heterophily):} Consider a neighborhood exhibiting high feature variance (e.g., class boundaries), such that $\|X_i - X_j\|$ is large. Here, $\exp(-\|X_i-X_j\|^2) \to 0$. The regularization term vanishes, yielding:
\begin{equation}
    (P(\phi^{\ast})\bm{\mu})_i \approx \frac{X_i + 0}{1 + 0} = X_i.
\end{equation}
In this regime, the structural constraint is relaxed. The projector effectively decouples node $i$ from its dissimilar neighbors, allowing the reconstructed signal to retain the high-frequency information residual essential for distinguishing heterophilic boundaries.

Since $\exp(-\|X_i-X_j\|^2)$ varies continuously, $P(\phi^{\ast})$ adaptively interpolates between these two extrema based on the local feature manifold.
\end{proof}


\section{Toy Experiment Results (Full)}\label{app:toyexp-full}

We provide the complete results of the toy experiments introduced in Section~\ref{sec:method-motivation}, complementing the summary reported in Table~\ref{table-toyexp}. 
The full results are shown in Table~\ref{table-toyexp-full}. 
As observed, on heterophilic graphs, adding the simple auxiliary loss yields substantial performance gains, with an average improvement exceeding \textbf{6 points} (see the ``$\Delta\uparrow$ (Heterophily)'' column). 
These results further strengthen the empirical motivation behind our method.

Interestingly, although the auxiliary loss is designed to address heterophily, we also observe minor improvements on homophilic datasets. 
This is consistent with the fact that many benchmarks labeled as ``homophilic''—such as those from the OGB suite—actually contain mixed patterns with non-negligible local heterophily~\citep{heterophily-gnn-survey,heterophily-gnn-survey-2,heterophily-gnn-survey-3}. 
Such latent heterophily allows the same enhancement to deliver modest benefits even on nominally homophilic graphs.

\begin{table*}[!th]
\caption{Full toy experiment results. We highlight substantial improvements—observed exclusively on heterophilic datasets—in \textbf{bold}. ``$\Delta\uparrow$ (Heterophily)'' reports the average gain on heterophilic graphs.}
\label{table-toyexp-full}
\centering
\setlength{\tabcolsep}{6pt}
\resizebox{\textwidth}{!}{
\begin{tabular}{ccccc|ccc}
\toprule
\multirow{2}{*}{Backbone} & \multirow{2}{*}{Method}        & \multicolumn{3}{c|}{Heterophilic Graphs}                                 & \multicolumn{2}{c}{Homophilic Graphs} & \multirow{2}{*}{$\Delta\uparrow$ (Heterophily)} \\ \cmidrule(lr){3-7}
                          &                                & Genius                 & Gamers                 & Pokec                  & Arxiv             & Products          &                                                 \\ \midrule
\multirow{7}{*}{GCN}      & Full Graph                     & 87.42                  & 62.18                  & 75.45                  & 71.74             & 75.64             & /                                               \\ \cmidrule(lr){2-8} 
                          & SCAL                           & 67.47                  & 47.39                  & 50.44                  & 63.42             & 68.37             & \multirow{2}{*}{\textbf{4.72}}                  \\
                          & SCAL w/ $\tilde{\mathcal{L}}$  & 71.32 (\textbf{+3.85}) & 52.86 (\textbf{+5.47}) & 55.29 (\textbf{+4.85}) & 64.71 (+1.29)     & 69.11 (+0.74)     &                                                 \\ \cmidrule(lr){2-8} 
                          & FGC                            & 70.13                  & 40.14                  & 53.81                  & 65.52             & 71.28             & \multirow{2}{*}{\textbf{6.02}}                  \\
                          & FGC w/ $\tilde{\mathcal{L}}$   & 75.67 (\textbf{+5.54}) & 49.26 (\textbf{+9.12}) & 57.22 (\textbf{+3.41}) & 66.18 (+0.66)     & 72.06 (+0.78)     &                                                 \\ \cmidrule(lr){2-8} 
                          & SGBGC                          & 71.22                  & 44.62                  & 55.48                  & 64.12             & 71.96             & \multirow{2}{*}{\textbf{6.00}}                  \\
                          & SGBGC w/ $\tilde{\mathcal{L}}$ & 73.81 (\textbf{+2.59}) & 53.09 (\textbf{+8.47}) & 62.42 (\textbf{+6.94}) & 65.72 (+1.60)     & 72.37 (+0.41)     &                                                 \\ \midrule
\multirow{7}{*}{LINKX}    & Full Graph                     & 90.77                  & 66.06                  & 82.04                  & 69.54             & 74.59             & /                                               \\ \cmidrule(lr){2-8} 
                          & SCAL                           & 60.62                  & 50.13                  & 56.48                  & 57.71             & 62.49             & \multirow{2}{*}{\textbf{7.14}}                  \\
                          & SCAL w/ $\tilde{\mathcal{L}}$  & 70.44 (\textbf{+9.82}) & 54.28 (\textbf{+4.15}) & 63.93 (\textbf{+7.45}) & 59.38 (+1.67)     & 63.61 (+1.12)     &                                                 \\ \cmidrule(lr){2-8} 
                          & FGC                            & 65.49                  & 42.48                  & 57.63                  & 59.38             & 67.52             & \multirow{2}{*}{\textbf{6.91}}                  \\
                          & FGC w/ $\tilde{\mathcal{L}}$   & 71.92 (\textbf{+6.43}) & 48.62 (\textbf{+6.14}) & 65.81 (\textbf{+8.18}) & 61.44 (+2.06)     & 68.77 (+1.25)     &                                                 \\ \cmidrule(lr){2-8} 
                          & SGBGC                          & 62.68                  & 47.60                  & 54.18                  & 61.55             & 64.93             & \multirow{2}{*}{\textbf{6.09}}                  \\
                          & SGBGC w/ $\tilde{\mathcal{L}}$ & 68.13 (\textbf{+5.45}) & 52.83 (\textbf{+5.23}) & 61.77 (\textbf{+7.59}) & 63.24 (+1.69)     & 66.19 (+1.26)     &                                                 \\ \midrule
\multirow{7}{*}{GloGNN}   & Full Graph                     & 90.66                  & 66.19                  & 83.00                  & 72.68             & 77.48             & /                                               \\ \cmidrule(lr){2-8} 
                          & SCAL                           & 59.47                  & 46.28                  & 61.58                  & 57.09             & 66.53             & \multirow{2}{*}{\textbf{6.97}}                  \\
                          & SCAL w/ $\tilde{\mathcal{L}}$  & 67.41 (\textbf{+7.94}) & 53.11 (\textbf{+6.83}) & 67.74 (\textbf{+6.16}) & 58.32 (+1.23)     & 67.44 (+0.91)     &                                                 \\ \cmidrule(lr){2-8} 
                          & FGC                            & 63.81                  & 45.92                  & 66.91                  & 61.73             & 70.46             & \multirow{2}{*}{\textbf{6.13}}                  \\
                          & FGC w/ $\tilde{\mathcal{L}}$   & 72.33 (\textbf{+8.52}) & 51.26 (\textbf{+5.34}) & 71.44 (\textbf{+4.53}) & 63.08 (+1.35)     & 71.52 (+1.06)     &                                                 \\ \cmidrule(lr){2-8} 
                          & SGBGC                          & 67.33                  & 47.15                  & 63.77                  & 61.24             & 69.11             & \multirow{2}{*}{\textbf{6.23}}                  \\
                          & SGBGC w/ $\tilde{\mathcal{L}}$ & 73.26 (\textbf{+5.93}) & 54.31 (\textbf{+7.17}) & 69.36 (\textbf{+5.59}) & 62.75 (+1.51)     & 70.07 (+0.96)     &                                                 \\ \bottomrule
\end{tabular}}
\end{table*}


\section{Experimental Details}\label{app:expdetails}
This section provides all supplementary experimental details, including dataset statistics (Appendix~\ref{app:expdetails-dataset}), descriptions of coarsening training baselines (Appendix~\ref{app:expdetails-coarsetrain}), backbone GNN introductions (Appendix~\ref{app:expdetails-backboneGNN}), and additional experimental results (Appendix~\ref{app:expdetails-moreexp}).


\subsection{Dataset Statistics}\label{app:expdetails-dataset}
All dataset statistics are summarized in Table~\ref{table-datasets-statistics}. 
To characterize homophily levels, we report the \textit{edge homophily} metric~\citep{H2GCN}, where a smaller value of $\mathcal{H}$ indicates stronger heterophily.

\begin{table*}[!t]
  \caption{Statistics for the large graph datasets. 
  $\mathcal{H}$ quantifies the homophily level~\citep{H2GCN}, where lower values denote stronger heterophily, and higher values indicate more homophilic graphs.}
  \label{table-datasets-statistics}
  \centering
    \begin{tabular}{cccccc|cc}
    \toprule
    \multirow{2}{*}{Statistics} & \multicolumn{5}{c|}{Heterophilic Graphs}                   & \multicolumn{2}{c}{Homophilic Graphs} \\ \cmidrule(lr){2-8} 
                                & Genius  & Gamers    & Pokec      & Snap       & arXiv-year & ogbn-arXiv       & ogbn-products      \\ \midrule
    \# Nodes                    & 421,961 & 168,114   & 1,632,803  & 2,923,922  & 169,343    & 169,343          & 2,449,029          \\
    \# Edges                    & 984,979 & 6,797,557 & 30,622,564 & 13,975,788 & 1,166,243  & 1,157,799        & 61,859,140         \\
    \# Features                 & 12      & 7         & 65         & 269        & 128        & 128              & 100                \\
    \# Classes                  & 2       & 2         & 2          & 5          & 5          & 40               & 47                 \\
    \# $\mathcal{H}$            & 0.62    & 0.45      & 0.44       & 0.07       & 0.22       & 0.66             & 0.82               \\ \bottomrule
    \end{tabular}
\end{table*}


\subsection{Overview of Coarsening Training Baselines}\label{app:expdetails-coarsetrain}
We provide an overview of the four advanced coarsening-based GNN training pipelines used in our experiments, summarized as follows:

\begin{itemize}[leftmargin=15pt,parsep=2pt,itemsep=2pt,topsep=2pt]
    \item \textbf{SCAL~\citep{GC-scal}.} SCAL introduces a spectral graph coarsening technique~\citep{Gcoarse-2} that operates in an unsupervised manner, relying solely on the inherent graph structure. 
    The coarsened graph is constructed during preprocessing, and GNNs are subsequently trained on these reduced graphs. 
    A novel coarsened graph convolution is proposed to maintain consistency between operations on the original and coarsened graphs. Optimization is performed exclusively with respect to the loss computed on the coarsened graph labels. 
    As the first work to explore GNN training on coarsened graphs, SCAL has inspired numerous follow-up studies aimed at improving computational efficiency and scalability.
    \item \textbf{FGC~\citep{Gcoarse-7-FGC}.} FGC enhances graph coarsening by incorporating both structural and feature information. 
    It formulates the coarsening process as the minimization of a Dirichlet energy functional, which unifies graph topology and node attributes. 
    The resulting coarsened graphs are more informative and support superior downstream GNN training.
    \item \textbf{UGC~\citep{UGC}.} UGC employs carefully designed hash functions for efficient graph coarsening on attributed graphs. 
    In addition, it explicitly accounts for the heterophily present in such graphs. 
    This combination allows for the generation of coarsened graphs that outperform those produced by FGC in terms of training effectiveness.
    \item \textbf{SGBGC~\citep{SGBGC}.} SGBGC constructs coarsened graphs via an iterative granulation process. 
    It partitions the original graph into granular-balls based on a node purity criterion, with each granular-ball treated as a supernode. 
    This approach achieves substantial graph compression while preserving structural integrity, offering notable gains in both efficiency and scalability.
\end{itemize}


\subsection{Backbone GNN Introductions}\label{app:expdetails-backboneGNN}
We briefly summarize the six backbone GNNs used in our experiments:

\begin{itemize}[leftmargin=15pt,parsep=2pt,itemsep=2pt,topsep=2pt]
    \item \textbf{GCN~\citep{GCN}.} GCN approximates spectral graph convolution via a first-order polynomial, aggregating normalized neighbor features in each layer. Its simplicity and efficiency make it a strong baseline, though its inherent smoothness bias limits performance on heterophilic graphs.
    \item \textbf{LINKX~\citep{dataset6-large-hetero}.} LINKX separately embeds node features and adjacency structure and fuses them using MLPs, avoiding graph convolution altogether. This decoupled design removes the homophily bias and scales well to large, low-homophily graphs where many GNNs struggle.
    \item \textbf{GloGNN~\citep{glognn++}.} GloGNN aggregates information using a learned global correlation matrix rather than local neighborhoods, enabling flexible (including signed) information mixing. A closed-form update with Woodbury acceleration allows linear-time scaling, making GloGNN highly effective on heterophilic graphs.
    \item \textbf{GPRGNN~\citep{GPRGNN}.} GPRGNN learns the coefficients of a generalized PageRank propagation, enabling adaptive multi-hop filtering. By learning propagation weights instead of fixing them, it naturally adjusts to both homophilic and heterophilic patterns and mitigates over-smoothing.
    \item \textbf{SGFormer~\citep{SGFormer}.} SGFormer employs a single-layer global attention paired with a lightweight GNN, capturing all-pair interactions in one step with linear complexity. It is extremely efficient but lacks the richer positional and structural encodings common in deeper graph Transformers.
    \item \textbf{Polynormer~\citep{Polynormer}.} Polynormer is a kernel-based graph Transformer that computes attention via graph kernels derived from sampled neighborhoods. Its polynomial-expressive architecture learns high-degree equivariant polynomials through a linear local-to-global attention mechanism.
\end{itemize}


\subsection{Supplementary Experimental Results}\label{app:expdetails-moreexp}
This section provides additional experimental results that were moved from the main text due to space constraints. 
Specifically, we include:
\begin{itemize}[leftmargin=15pt,parsep=2pt,itemsep=2pt,topsep=2pt]
    \item {\bf Table~\ref{table-mainexp-1-std}:} Standard deviations corresponding to the mean results in Table~\ref{table-mainexp-1}.
    \item {\bf Table~\ref{table-mainexp-2}:} ACE-integrated results for the remaining backbone GNNs: GPRGNN, SGFormer, and Polynormer.
    \item {\bf Table~\ref{table-convmatch-gcond}:} Results on \textbf{task-optimized} coarsening training pipelines—ConvMatch and GCond—using GCN and GPRGNN across all benchmarks.
    \item {\bf Table~\ref{table-0.01coarse}:} Results under a 1\% coarsening ratio, using GCN and GPRGNN as backbones.
    \item {\bf Table~\ref{table-againstSOTA-GPRGNN}:} Comparison between ACE-enhanced coarsening methods and state-of-the-art approaches under GPRGNN backbone.
    \item {\bf Table~\ref{table-ablation-full-GCN} and Table~\ref{table-ablation-full-GPRGNN}:} Comprehensive ablation study of ACE components under the GCN and GPRGNN backbones.
    \item {\bf Figure~\ref{fig:beta-ablation-full}:} Complete ablation results for the tradeoff coefficient $\beta$.
    \item {\bf Figure~\ref{fig:preprocessing-full}:} Full preprocessing-time comparison between vanilla and ACE-enhanced coarsening training pipelines.
    \item {\bf Figure~\ref{fig:training-time-full}:} Full per-epoch training-time comparison across coarsening training pipelines.
    \item {\bf Figure~\ref{fig:gpu-memory-full}:} Full training memory consumption comparison across coarsening training pipelines.
    \item {\bf Figure~\ref{fig:noiserobust}:} Assessment of feature noise robustness for ACE and state-of-the-art methods.
\end{itemize}
\begin{table*}[!t]
\caption{Standard deviations corresponding to the mean results reported in Table~\ref{table-mainexp-1}.}
\label{table-mainexp-1-std}
\centering
\setlength{\tabcolsep}{5pt}
\resizebox{.8\textwidth}{!}{
\begin{tabular}{ccccccc|cc}
\toprule
\multirow{2}{*}{Backbone} & \multirow{2}{*}{Method} & \multicolumn{5}{c|}{Heterophilic Graphs}    & \multicolumn{2}{c}{Homophilic Graphs} \\ \cmidrule(lr){3-9} 
                          &                         & Genius & Gamers & Pokec & Snap & arXiv-year & ogbn-arXiv       & ogbn-products      \\ \midrule
\multirow{9}{*}{GCN}      & -                       & 0.34   & 0.29   & 0.21  & 0.22 & 0.38       & 0.23             & 0.26               \\ \cmidrule(lr){2-9} 
                          & SCAL                    & 0.18   & 0.35   & 0.23  & 0.26 & 0.26       & 0.23             & 0.23               \\
                          & + ACE                   & 0.21   & 0.29   & 0.21  & 0.27 & 0.22       & 0.22             & 0.38               \\ \cmidrule(lr){2-9} 
                          & FGC                     & 0.18   & 0.23   & 0.25  & 0.23 & 0.25       & 0.21             & 0.18               \\
                          & + ACE                   & 0.24   & 0.18   & 0.31  & 0.21 & 0.27       & 0.38             & 0.22               \\ \cmidrule(lr){2-9} 
                          & UGC                     & 0.27   & 0.26   & 0.24  & 0.24 & 0.23       & 0.22             & 0.20               \\
                          & + ACE                   & 0.31   & 0.24   & 0.21  & 0.26 & 0.22       & 0.21             & 0.23               \\ \cmidrule(lr){2-9} 
                          & SGBGC                   & 0.18   & 0.20   & 0.25  & 0.19 & 0.23       & 0.20             & 0.23               \\
                          & + ACE                   & 0.22   & 0.26   & 0.18  & 0.23 & 0.25       & 0.23             & 0.29               \\ \midrule
\multirow{9}{*}{LINKX}    & -                       & 0.33   & 0.19   & 0.26  & 0.29 & 0.32       & 0.21             & 0.19               \\ \cmidrule(lr){2-9} 
                          & SCAL                    & 0.32   & 0.21   & 0.17  & 0.33 & 0.29       & 0.20             & 0.23               \\
                          & + ACE                   & 0.29   & 0.19   & 0.27  & 0.32 & 0.21       & 0.29             & 0.27               \\ \cmidrule(lr){2-9} 
                          & FGC                     & 0.24   & 0.21   & 0.25  & 0.20 & 0.25       & 0.23             & 0.31               \\
                          & + ACE                   & 0.32   & 0.24   & 0.32  & 0.26 & 0.32       & 0.20             & 0.23               \\ \cmidrule(lr){2-9} 
                          & UGC                     & 0.19   & 0.26   & 0.31  & 0.17 & 0.38       & 0.17             & 0.38               \\
                          & + ACE                   & 0.25   & 0.31   & 0.17  & 0.23 & 0.23       & 0.23             & 0.25               \\ \cmidrule(lr){2-9} 
                          & SGBGC                   & 0.17   & 0.22   & 0.26  & 0.20 & 0.27       & 0.20             & 0.19               \\
                          & + ACE                   & 0.33   & 0.17   & 0.25  & 0.32 & 0.23       & 0.33             & 0.23               \\ \midrule
\multirow{9}{*}{GloGNN}   & -                       & 0.32   & 0.23   & 0.27  & 0.17 & 0.23       & 0.33             & 0.26               \\ \cmidrule(lr){2-9} 
                          & SCAL                    & 0.33   & 0.33   & 0.19  & 0.23 & 0.32       & 0.29             & 0.33               \\
                          & + ACE                   & 0.24   & 0.26   & 0.32  & 0.20 & 0.26       & 0.27             & 0.38               \\ \cmidrule(lr){2-9} 
                          & FGC                     & 0.19   & 0.22   & 0.27  & 0.32 & 0.31       & 0.33             & 0.20               \\
                          & + ACE                   & 0.27   & 0.32   & 0.33  & 0.25 & 0.27       & 0.31             & 0.26               \\ \cmidrule(lr){2-9} 
                          & UGC                     & 0.19   & 0.26   & 0.17  & 0.26 & 0.20       & 0.29             & 0.29               \\
                          & + ACE                   & 0.27   & 0.17   & 0.38  & 0.19 & 0.27       & 0.31             & 0.33               \\ \cmidrule(lr){2-9} 
                          & SGBGC                   & 0.22   & 0.22   & 0.17  & 0.24 & 0.20       & 0.19             & 0.26               \\
                          & + ACE                   & 0.24   & 0.29   & 0.29  & 0.19 & 0.26       & 0.29             & 0.27               \\ \bottomrule
\end{tabular}}
\end{table*}
\begin{table*}[!t]
\caption{Results of integrating ACE into GPRGNN, SGFormer, and Polynormer. Improvements on heterophilic datasets appear in \textbf{bold}.}
\label{table-mainexp-2}
\centering
\setlength{\tabcolsep}{4pt}
\resizebox{\textwidth}{!}{
\begin{tabular}{ccccccc|cc}
\toprule
\multirow{2}{*}{Backbone}   & \multirow{2}{*}{Method} & \multicolumn{5}{c|}{Heterophilic Graphs}                                                                                        & \multicolumn{2}{c}{Homophilic Graphs} \\ \cmidrule(lr){3-9} 
                            &                         & Genius                  & Gamers                  & Pokec                   & Snap                    & arXiv-year              & ogbn-arXiv        & ogbn-products     \\ \midrule
\multirow{9}{*}{GPRGNN}     & -                       & 90.05                   & 61.89                   & 78.83                   & 40.19                   & 45.07                   & 71.78             & 78.29             \\ \cmidrule(lr){2-9} 
                            & SCAL                    & 66.19                   & 43.48                   & 60.92                   & 28.22                   & 26.53                   & 63.28             & 65.92             \\
                            & + ACE                   & 78.57 (\textbf{+12.38}) & 51.74 (\textbf{+8.26})  & 70.25 (\textbf{+9.33})  & 32.42 (\textbf{+4.20})  & 35.59 (\textbf{+9.06})  & 64.76 (+1.48)     & 68.25 (+2.33)     \\ \cmidrule(lr){2-9} 
                            & FGC                     & 69.28                   & 45.92                   & 57.14                   & 26.44                   & 28.60                   & 62.19             & 63.28             \\
                            & + ACE                   & 79.02 (\textbf{+9.74})  & 52.38 (\textbf{+6.46})  & 64.87 (\textbf{+7.73})  & 31.22 (\textbf{+4.78})  & 36.73 (\textbf{+8.13})  & 64.29 (+2.10)     & 65.83 (+2.55)     \\ \cmidrule(lr){2-9} 
                            & UGC                     & 68.06                   & 41.20                   & 61.33                   & 25.38                   & 29.12                   & 61.44             & 65.26             \\
                            & + ACE                   & 78.17 (\textbf{+10.11}) & 51.06 (\textbf{+9.86})  & 69.37 (\textbf{+8.04})  & 31.46 (\textbf{+6.08})  & 34.59 (\textbf{+5.47})  & 63.19 (+1.75)     & 67.84 (+2.58)     \\ \cmidrule(lr){2-9} 
                            & SGBGC                   & 69.93                   & 44.26                   & 62.41                   & 26.93                   & 31.26                   & 63.14             & 63.02             \\
                            & + ACE                   & 77.39 (\textbf{+7.46})  & 50.61 (\textbf{+6.35})  & 67.38 (\textbf{+4.97})  & 30.24 (\textbf{+3.31})  & 38.42 (\textbf{+7.16})  & 65.02 (+1.88)     & 68.47 (+5.45)     \\ \midrule
\multirow{9}{*}{SGFormer}   & -                       & 85.01                   & 65.93                   & 80.84                   & 63.84                   & 51.23                   & 72.63             & 81.54             \\ \cmidrule(lr){2-9} 
                            & SCAL                    & 63.91                   & 51.74                   & 63.47                   & 42.73                   & 30.81                   & 61.38             & 71.74             \\
                            & + ACE                   & 75.81 (\textbf{+11.90}) & 57.38 (\textbf{+5.64})  & 71.88 (\textbf{+8.41})  & 53.62 (\textbf{+10.89}) & 43.76 (\textbf{+12.95}) & 63.02 (+1.64)     & 73.84 (+2.10)     \\ \cmidrule(lr){2-9} 
                            & FGC                     & 67.26                   & 52.09                   & 61.37                   & 40.29                   & 32.48                   & 63.81             & 73.62             \\
                            & + ACE                   & 76.31 (\textbf{+9.05})  & 56.38 (\textbf{+4.29})  & 70.33 (\textbf{+8.96})  & 50.13 (\textbf{+9.84})  & 43.02 (\textbf{+10.54}) & 65.11 (+1.30)     & 74.84 (+1.22)     \\ \cmidrule(lr){2-9} 
                            & UGC                     & 65.38                   & 49.07                   & 64.35                   & 44.16                   & 28.60                   & 60.77             & 70.04             \\
                            & + ACE                   & 75.48 (\textbf{+10.10}) & 58.34 (\textbf{+9.27})  & 71.48 (\textbf{+7.13})  & 49.78 (\textbf{+5.62})  & 38.26 (\textbf{+9.66})  & 63.29 (+2.52)     & 73.16 (+3.12)     \\ \cmidrule(lr){2-9} 
                            & SGBGC                   & 68.73                   & 52.66                   & 60.73                   & 46.38                   & 34.01                   & 61.62             & 68.30             \\
                            & + ACE                   & 76.12 (\textbf{+7.39})  & 57.02 (\textbf{+4.36})  & 69.66 (\textbf{+8.93})  & 52.09 (\textbf{+5.71})  & 39.74 (\textbf{+5.73})  & 64.21 (+2.59)     & 72.73 (+4.43)     \\ \midrule
\multirow{9}{*}{Polynormer} & -                       & 85.64                   & 64.72                   & 86.06                   & 61.92                   & 53.48                   & 73.40             & 83.82             \\ \cmidrule(lr){2-9} 
                            & SCAL                    & 64.93                   & 48.42                   & 64.38                   & 40.63                   & 32.74                   & 63.93             & 71.38             \\
                            & + ACE                   & 75.73 (\textbf{+10.80}) & 55.62 (\textbf{+7.20})  & 73.16 (\textbf{+8.78})  & 47.63 (\textbf{+7.00})  & 40.61 (\textbf{+7.87})  & 65.28 (+1.35)     & 73.47 (+2.09)     \\ \cmidrule(lr){2-9} 
                            & FGC                     & 62.03                   & 44.22                   & 66.63                   & 38.66                   & 35.26                   & 64.22             & 72.48             \\
                            & + ACE                   & 74.79 (\textbf{+12.76}) & 57.48 (\textbf{+13.26}) & 77.81 (\textbf{+11.18}) & 50.13 (\textbf{+11.47}) & 46.72 (\textbf{+11.46}) & 65.79 (+1.57)     & 74.61 (+2.13)     \\ \cmidrule(lr){2-9} 
                            & UGC                     & 66.02                   & 45.49                   & 68.49                   & 43.84                   & 30.34                   & 61.28             & 71.79             \\
                            & + ACE                   & 74.20 (\textbf{+8.18})  & 57.02 (\textbf{+11.53}) & 76.24 (\textbf{+7.75})  & 51.26 (\textbf{+7.42})  & 43.28 (\textbf{+12.94}) & 64.10 (+2.82)     & 73.45 (+1.66)     \\ \cmidrule(lr){2-9} 
                            & SGBGC                   & 67.98                   & 41.37                   & 62.13                   & 45.01                   & 32.63                   & 63.33             & 70.16             \\
                            & + ACE                   & 77.83 (\textbf{+9.85})  & 53.81 (\textbf{+12.44}) & 74.86 (\textbf{+12.73}) & 49.38 (\textbf{+4.37})  & 38.97 (\textbf{+6.34})  & 65.03 (+1.70)     & 72.88 (+2.72)     \\ \bottomrule
\end{tabular}}
\end{table*}
\begin{table*}[!t]
\caption{Additional results for \textbf{task-optimized} coarsening pipelines—GCond~\citep{Gcond} and ConvMatch~\citep{GC-GCNonly-convmatch}—evaluated using GCN (homophily-oriented) and GPRGNN (heterophily-oriented) backbones across diverse benchmarks.}
\label{table-convmatch-gcond}
\centering
\setlength{\tabcolsep}{4pt}
\resizebox{\textwidth}{!}{
\begin{tabular}{ccccccc|cc}
\toprule
\multirow{2}{*}{Backbone} & \multirow{2}{*}{Method} & \multicolumn{5}{c|}{Heterophilic Graphs}                                                                                     & \multicolumn{2}{c}{Homophilic Graphs} \\ \cmidrule(lr){3-9} 
                          &                         & Genius                  & Gamers                 & Pokec                   & Snap                   & arXiv-year             & ogbn-arXiv        & ogbn-products     \\ \midrule
\multirow{5}{*}{GCN}      & -                       & 87.42                   & 62.18                  & 75.45                   & 45.65                  & 46.02                  & 71.74             & 75.64             \\ \cmidrule(lr){2-9} 
                          & GCond                   & 70.22                   & 51.27                  & 53.77                   & 29.78                  & 30.22                  & 66.38             & 71.19             \\
                          & + ACE                   & 78.33 (\textbf{+8.11})  & 56.26 (\textbf{+4.99}) & 58.49 (\textbf{+4.72})  & 36.37 (\textbf{+6.59}) & 37.83 (\textbf{+7.61}) & 67.15 (+0.77)     & 72.01 (+0.82)     \\ \cmidrule(lr){2-9} 
                          & ConvMatch              & 68.13                   & 43.74                  & 50.66                   & 24.37                  & 25.94                  & 66.12             & 70.93             \\
                          & + ACE                   & 76.44 (\textbf{+8.31})  & 52.07 (\textbf{+8.33}) & 55.39 (\textbf{+4.73})  & 32.29 (\textbf{+7.92}) & 30.85 (\textbf{+4.91}) & 67.95 (+1.83)     & 72.83 (+1.90)     \\ \midrule
\multirow{5}{*}{GPRGNN}   & -                       & 90.05                   & 61.89                  & 78.83                   & 40.19                  & 45.07                  & 71.78             & 78.29             \\ \cmidrule(lr){2-9} 
                          & GCond                   & 71.39                   & 49.47                  & 58.65                   & 30.47                  & 31.38                  & 65.93             & 70.27             \\
                          & + ACE                   & 79.63 (\textbf{+8.24})  & 55.65 (\textbf{+6.18}) & 71.33 (\textbf{+12.68}) & 34.28 (\textbf{+3.81}) & 37.19 (\textbf{+5.81}) & 66.49 (+0.56)     & 72.01 (+1.74)     \\ \cmidrule(lr){2-9} 
                          & ConvMatch              & 66.31                   & 44.67                  & 51.40                   & 22.33                  & 24.60                  & 66.19             & 67.28             \\
                          & + ACE                   & 77.24 (\textbf{+10.93}) & 50.42 (\textbf{+5.75}) & 66.67 (\textbf{+15.27}) & 30.05 (\textbf{+7.72}) & 31.43 (\textbf{+6.83}) & 67.43 (+1.24)     & 68.44 (+1.19)     \\ \bottomrule
\end{tabular}}
\end{table*}
\begin{table*}[!t]
\caption{Additional results under an extreme 1\% coarsening ratio.}
\label{table-0.01coarse}
\centering
\setlength{\tabcolsep}{4pt}
\resizebox{\textwidth}{!}{
\begin{tabular}{ccccccc|cc}
\toprule
\multirow{2}{*}{Backbone} & \multirow{2}{*}{Method} & \multicolumn{5}{c|}{Heterophilic Graphs}                                                                                      & \multicolumn{2}{c}{Homophilic Graphs} \\ \cmidrule(lr){3-9} 
                          &                         & Genius                  & Gamers                  & Pokec                   & Snap                   & arXiv-year             & ogbn-arXiv        & ogbn-products     \\ \midrule
\multirow{7}{*}{GCN}      & -                       & 87.42                   & 62.18                   & 75.45                   & 45.65                  & 46.02                  & 71.74             & 75.64             \\ \cmidrule(lr){2-9} 
                          & SCAL                    & 55.42                   & 39.21                   & 46.26                   & 20.41                  & 22.43                  & 60.39             & 64.69             \\
                          & + ACE                   & 73.76 (\textbf{+18.34}) & 52.43 (\textbf{+13.22}) & 53.80 (\textbf{+7.54})  & 28.73 (\textbf{+8.32}) & 29.38 (\textbf{+6.95}) & 64.47 (+4.08)     & 68.73 (+4.04)     \\ \cmidrule(lr){2-9} 
                          & FGC                     & 63.84                   & 36.08                   & 45.37                   & 21.23                  & 21.40                  & 62.86             & 66.14             \\
                          & + ACE                   & 72.74 (\textbf{+8.90})  & 50.18 (\textbf{+14.10}) & 54.55 (\textbf{+9.18})  & 30.25 (\textbf{+9.02}) & 29.01 (\textbf{+7.61}) & 63.71 (+0.85)     & 70.22 (+4.08)     \\ \cmidrule(lr){2-9} 
                          & SGBGC                   & 64.17                   & 40.25                   & 48.66                   & 23.24                  & 22.42                  & 62.56             & 64.81             \\
                          & + ACE                   & 71.02 (\textbf{+6.85})  & 52.39 (\textbf{+12.14}) & 53.29 (\textbf{+4.63})  & 28.63 (\textbf{+5.39}) & 30.22 (\textbf{+7.80}) & 64.66 (+2.10)     & 68.37 (+3.56)     \\ \midrule
\multirow{7}{*}{GPRGNN}   & -                       & 90.05                   & 61.89                   & 78.83                   & 40.19                  & 45.07                  & 71.78             & 78.29             \\ \cmidrule(lr){2-9} 
                          & SCAL                    & 59.39                   & 37.82                   & 50.68                   & 26.28                  & 22.15                  & 61.47             & 63.06             \\
                          & + ACE                   & 73.54 (\textbf{+14.15}) & 47.46 (\textbf{+9.64})  & 61.62 (\textbf{+10.94}) & 30.33 (\textbf{+4.05}) & 30.42 (\textbf{+8.27}) & 63.22 (+1.75)     & 65.81 (+2.75)     \\ \cmidrule(lr){2-9} 
                          & FGC                     & 57.47                   & 36.81                   & 51.42                   & 23.37                  & 26.55                  & 61.05             & 60.39             \\
                          & + ACE                   & 71.23 (\textbf{+13.76}) & 45.23 (\textbf{+8.42})  & 59.83 (\textbf{+8.41})  & 27.83 (\textbf{+4.46}) & 31.17 (\textbf{+4.62}) & 63.14 (+2.09)     & 63.77 (+3.38)     \\ \cmidrule(lr){2-9} 
                          & SGBGC                   & 61.63                   & 38.41                   & 54.48                   & 23.41                  & 25.07                  & 61.39             & 60.97             \\
                          & + ACE                   & 70.78 (\textbf{+9.15})  & 47.81 (\textbf{+9.40})  & 64.13 (\textbf{+9.65})  & 28.35 (\textbf{+4.94}) & 30.93 (\textbf{+5.86}) & 63.77 (+2.38)     & 65.48 (+4.51)     \\ \bottomrule
\end{tabular}}
\end{table*}
\begin{table}[!th]
\centering
\caption{Comparison with state-of-the-art scalable GNN training methods based on graph data reduction. Results on a \textbf{GPRGNN} backbone with a 10\% reduction ratio are reported in terms of test accuracy and total training time.}
\resizebox{\linewidth}{!}{
\begin{tabular}{cccc|cc|cc|cc}
\toprule
\multirow{2}{*}{Category}      & \multirow{2}{*}{Method} & \multicolumn{2}{c|}{Pokec} & \multicolumn{2}{c|}{Snap} & \multicolumn{2}{c|}{Gamers} & \multicolumn{2}{c}{Ogbn-product} \\ \cmidrule(lr){3-10} 
                               &                         & ACC       & Time (min)     & ACC       & Time (min)    & ACC        & Time (min)     & ACC          & Time (min)        \\ \midrule
\multirow{2}{*}{Coarsening}    & FGC                     & 57.14     & 22.48          & 26.44     & 30.84         & 45.92      & 10.28          & 63.28        & 32.17             \\
                               & UGC                     & 61.33     & 21.55          & 25.38     & 28.26         & 41.20      & 10.56          & 65.26        & 30.46             \\ \midrule
Data-Adaptive                  & GECC                    & 64.19     & 32.16          & 30.81     & 36.43         & 49.62      & 16.81          & 66.77        & 43.49             \\ \midrule
Subgraph-Sampling                       & AGS                     & 62.77     & 19.92          & 28.34     & 28.38         & 48.13      & 9.44           & 66.11        & 29.54             \\ \midrule
\multirow{2}{*}{\textbf{Ours}} & FGC+ACE                 & \textbf{64.87}     & 22.96          & \textbf{31.22}     & 31.37         & \textbf{52.38}      & 10.86          & 65.83        & 33.09             \\
                               & UGC+ACE                 & \textbf{69.37}     & 21.88          & \textbf{31.46}     & 29.84         & \textbf{51.06}      & 11.16          & \textbf{67.84}        & 31.22             \\ \bottomrule
\end{tabular}}
\label{table-againstSOTA-GPRGNN}
\end{table}
\begin{table}[!th]
\centering
\caption{Detailed ablation study of individual components in ACE. Results are reported on a \textbf{GCN} backbone with a 10\% reduction ratio, including test accuracy for the full method (ACE base) and relative performance drop for each variant.}
\begin{tabular}{ccccccc|c}
\toprule
Method               & Variant                  & Genius & Gamers & Pokec & Snap  & Ogbn-Product & Average Drop   \\ \midrule
\multirow{5}{*}{FGC} & ACE (base)               & 77.32  & 52.52  & 61.67 & 32.73 & 72.47        & -              \\ \cmidrule(lr){2-8} 
                     & ACE (fixed weighting)    & -0.35  & -0.59  & -0.41 & -0.62 & -0.06        & -0.41          \\ \cmidrule(lr){2-8} 
                     & ACE (no anisotropic reg) & -0.70  & -1.83  & -1.90 & -1.30 & -0.18        & -1.18          \\ \cmidrule(lr){2-8} 
                     & ACE (fixed projector)    & -1.34  & -2.60  & -3.51 & -1.66 & -0.36        & \textbf{-1.90} \\ \cmidrule(lr){2-8} 
                     & naive auxiliary loss     & -1.65  & -3.26  & -4.45 & -2.07 & -0.41        & \textbf{-2.37} \\ \midrule
\multirow{5}{*}{UGC} & ACE (base)               & 75.28  & 53.83  & 59.37 & 30.82 & 70.19        & -              \\ \cmidrule(lr){2-8} 
                     & ACE (fixed weighting)    & -0.26  & -0.71  & -0.63 & -0.50 & -0.11        & -0.44          \\ \cmidrule(lr){2-8} 
                     & ACE (no anisotropic reg) & -0.81  & -2.06  & -1.75 & -1.22 & -0.19        & -1.20          \\ \cmidrule(lr){2-8} 
                     & ACE (fixed projector)    & -1.72  & -3.61  & -3.83 & -1.93 & -0.35        & \textbf{-2.29} \\ \cmidrule(lr){2-8} 
                     & naive auxiliary loss     & -2.44  & -4.23  & -4.16 & -2.65 & -0.56        & \textbf{-2.70} \\ \bottomrule
\end{tabular}
\label{table-ablation-full-GCN}
\end{table}
\begin{table}[!th]
\centering
\caption{Detailed ablation study of individual components in ACE. Results are reported on a \textbf{GPRGNN} backbone with a 10\% reduction ratio, including test accuracy for the full method (ACE base) and relative performance drop for each variant.}
\begin{tabular}{ccccccc|c}
\toprule
Method               & Variant                  & Genius & Gamers & Pokec & Snap  & Ogbn-Product & Average Drop   \\ \midrule
\multirow{5}{*}{FGC} & ACE (base)               & 79.02  & 52.38  & 64.87 & 31.22 & 65.83        & -              \\ \cmidrule(lr){2-8} 
                     & ACE (fixed weighting)    & -0.27  & -0.41  & -0.46 & -0.44 & -0.10        & -0.34          \\ \cmidrule(lr){2-8} 
                     & ACE (no anisotropic reg) & -0.79  & -1.23  & -1.64 & -1.14 & -0.15        & -0.99          \\ \cmidrule(lr){2-8} 
                     & ACE (fixed projector)    & -1.65  & -2.84  & -3.47 & -1.92 & -0.28        & \textbf{-2.03} \\ \cmidrule(lr){2-8} 
                     & naive auxiliary loss     & -2.12  & -3.33  & -3.93 & -2.47 & -0.43        & \textbf{-2.46} \\ \midrule
\multirow{5}{*}{UGC} & ACE (base)               & 78.17  & 51.06  & 69.37 & 31.46 & 67.84        & -              \\ \cmidrule(lr){2-8} 
                     & ACE (fixed weighting)    & -0.38  & -0.56  & -0.42 & -0.53 & -0.08        & -0.39          \\ \cmidrule(lr){2-8} 
                     & ACE (no anisotropic reg) & -0.75  & -1.78  & -1.93 & -1.42 & -0.16        & -1.21          \\ \cmidrule(lr){2-8} 
                     & ACE (fixed projector)    & -2.14  & -3.59  & -3.45 & -2.46 & -0.30        & \textbf{-2.39} \\ \cmidrule(lr){2-8} 
                     & naive auxiliary loss     & -2.81  & -4.02  & -4.09 & -3.16 & -0.45        & \textbf{-2.91} \\ \bottomrule
\end{tabular}
\label{table-ablation-full-GPRGNN}
\end{table}
\begin{figure*}[!t]
  \centering
    \subfloat[SCAL on Pokec.]{
    \label{fig:beta-ablation-full-scal-pokec}
    \includegraphics[width=0.235\linewidth]{figures/beta_scal_pokec.pdf}
    }
    \subfloat[SCAL on Snap.]{
    \label{fig:beta-ablation-full-scal-snap}
    \includegraphics[width=0.235\linewidth]{figures/beta_scal_snap.pdf}
    }
    \subfloat[FGC on Pokec.]{
    \label{fig:beta-ablation-full-fgc-pokec}
    \includegraphics[width=0.235\linewidth]{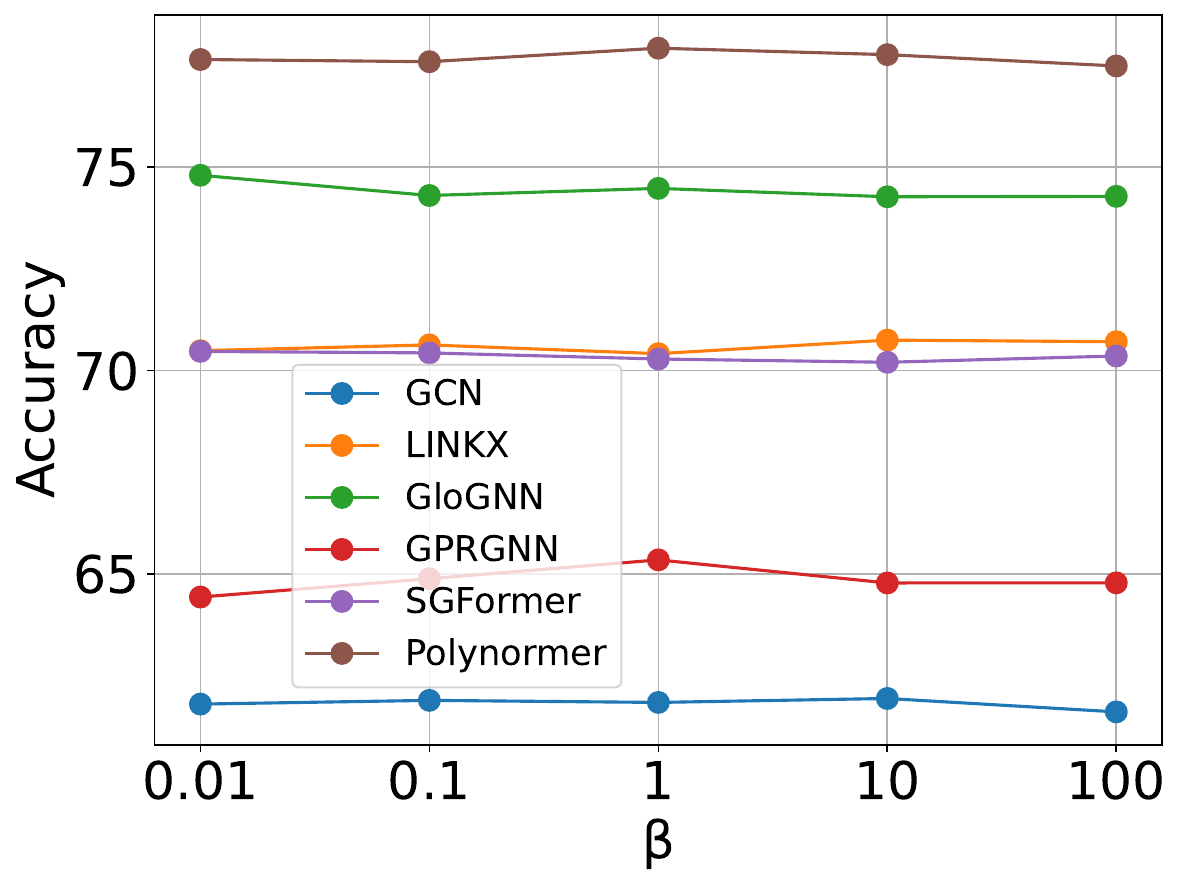}
    }
    \subfloat[FGC on Snap.]{
    \label{fig:beta-ablation-full-fgc-snap}
    \includegraphics[width=0.235\linewidth]{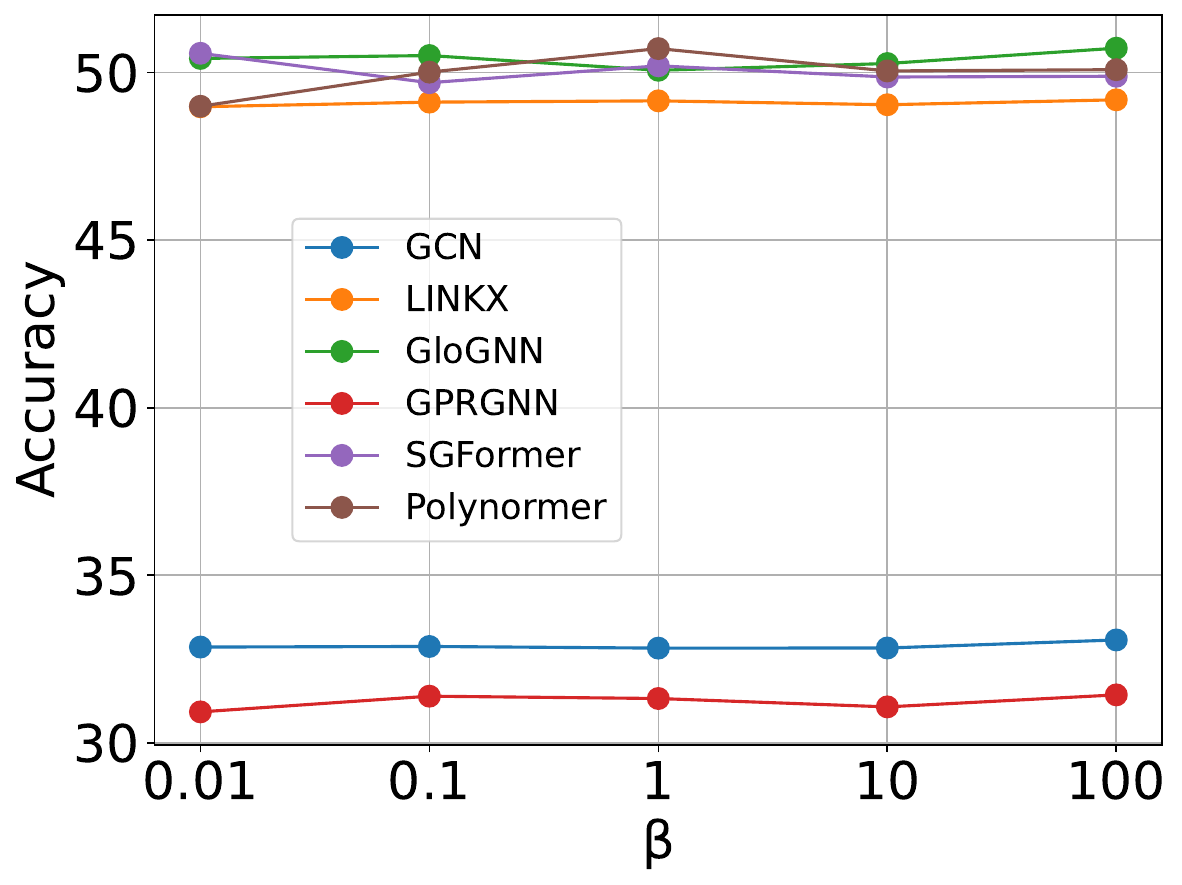}
    } \\
    \subfloat[UGC on Pokec.]{
    \label{fig:beta-ablation-full-ugc-pokec}
    \includegraphics[width=0.235\linewidth]{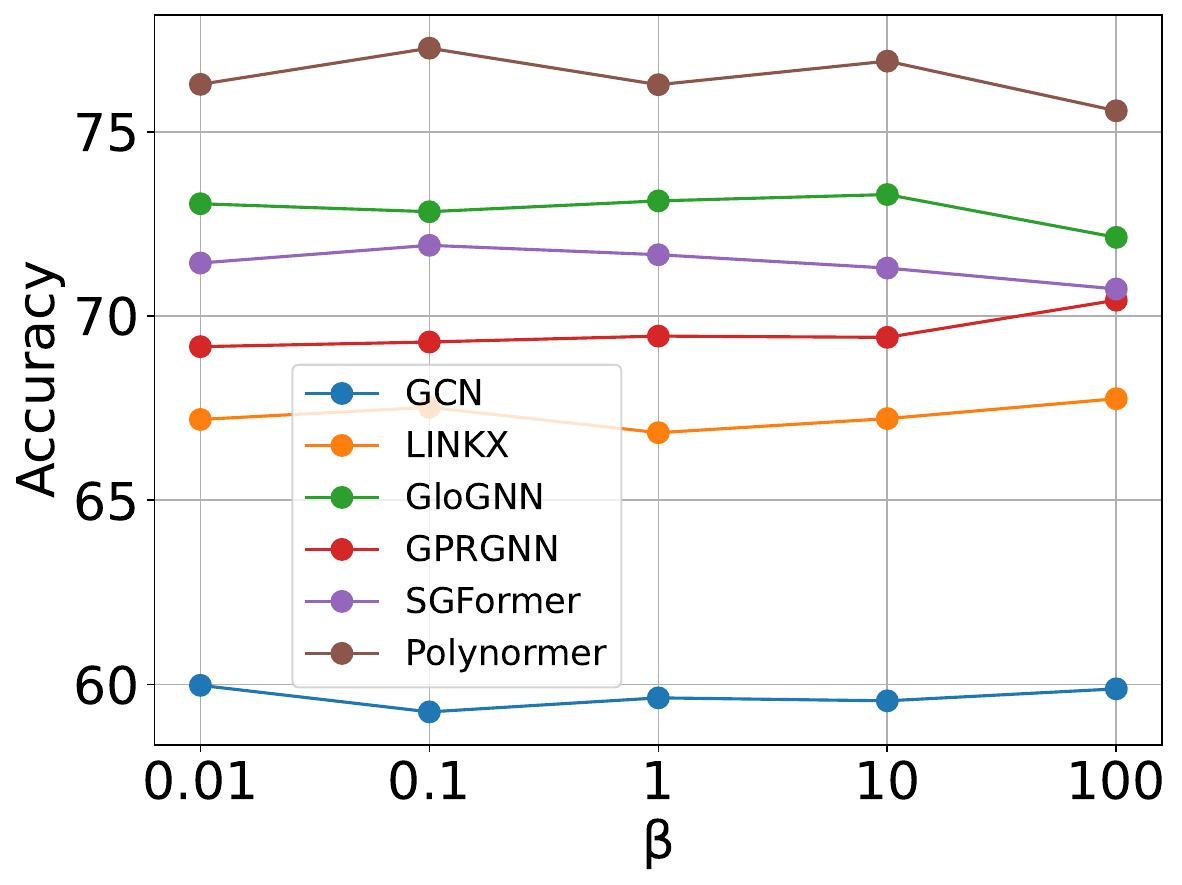}
    }
    \subfloat[UGC on Snap.]{
    \label{fig:beta-ablation-full-ugc-snap}
    \includegraphics[width=0.235\linewidth]{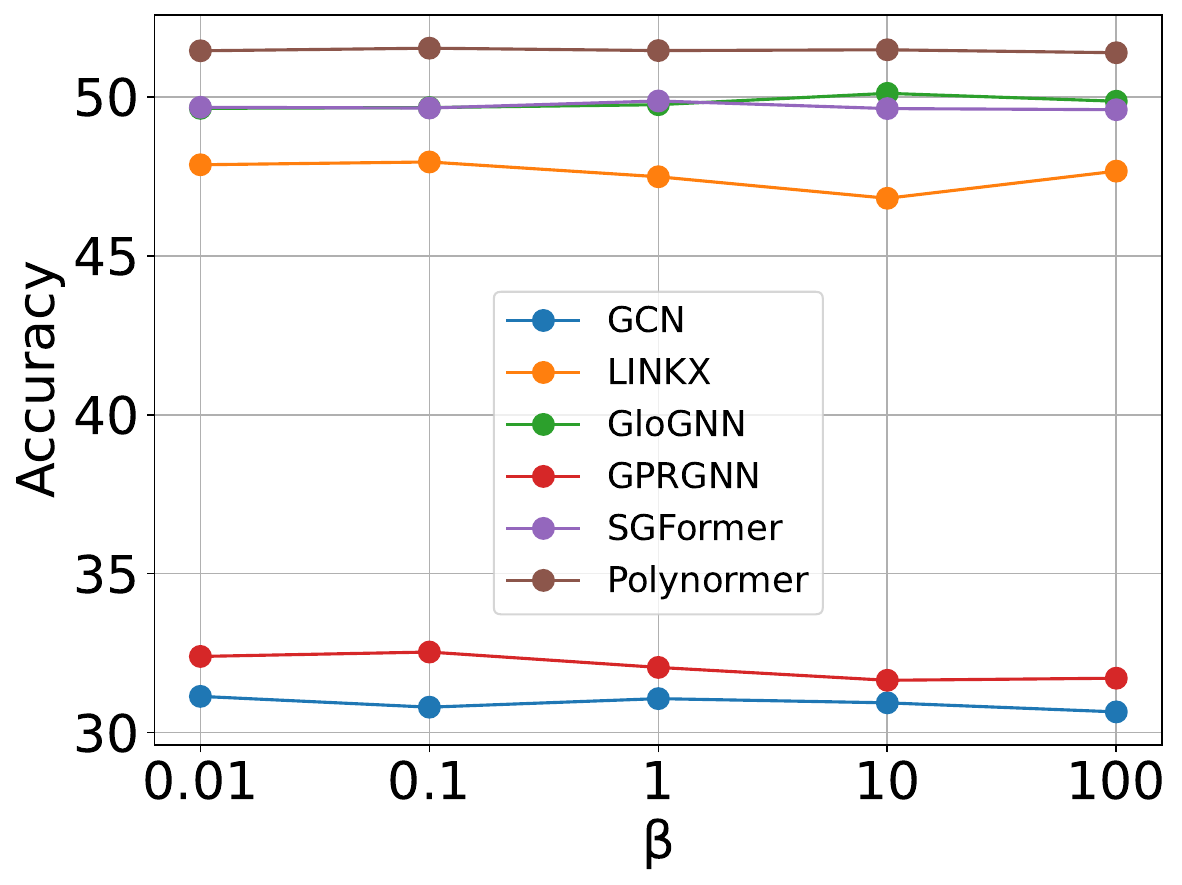}
    }
    \subfloat[SGBGC on Pokec.]{
    \label{fig:beta-ablation-full-sgbgc-pokec}
    \includegraphics[width=0.235\linewidth]{figures/beta_sgbgc_pokec.pdf}
    }
    \subfloat[SGBGC on Snap.]{
    \label{fig:beta-ablation-full-sgbgc-snap}
    \includegraphics[width=0.235\linewidth]{figures/beta_sgbgc_snap.pdf}
    }
  \caption{Effect of $\beta$ across coarsening pipelines, GNN backbones, and datasets.}
  \label{fig:beta-ablation-full}
\end{figure*}
\begin{figure*}[!t]
  \centering
    \subfloat[Preprocessing time: SCAL vs. SCAL+ACE.]{
    \label{fig:preprocessing-full-scal}
    \includegraphics[width=0.495\linewidth]{figures/scal-preprocess.pdf}
    }
    \subfloat[Preprocessing time: FGC vs. FGC+ACE.]{
    \label{fig:preprocessing-full-FGC}
    \includegraphics[width=0.495\linewidth]{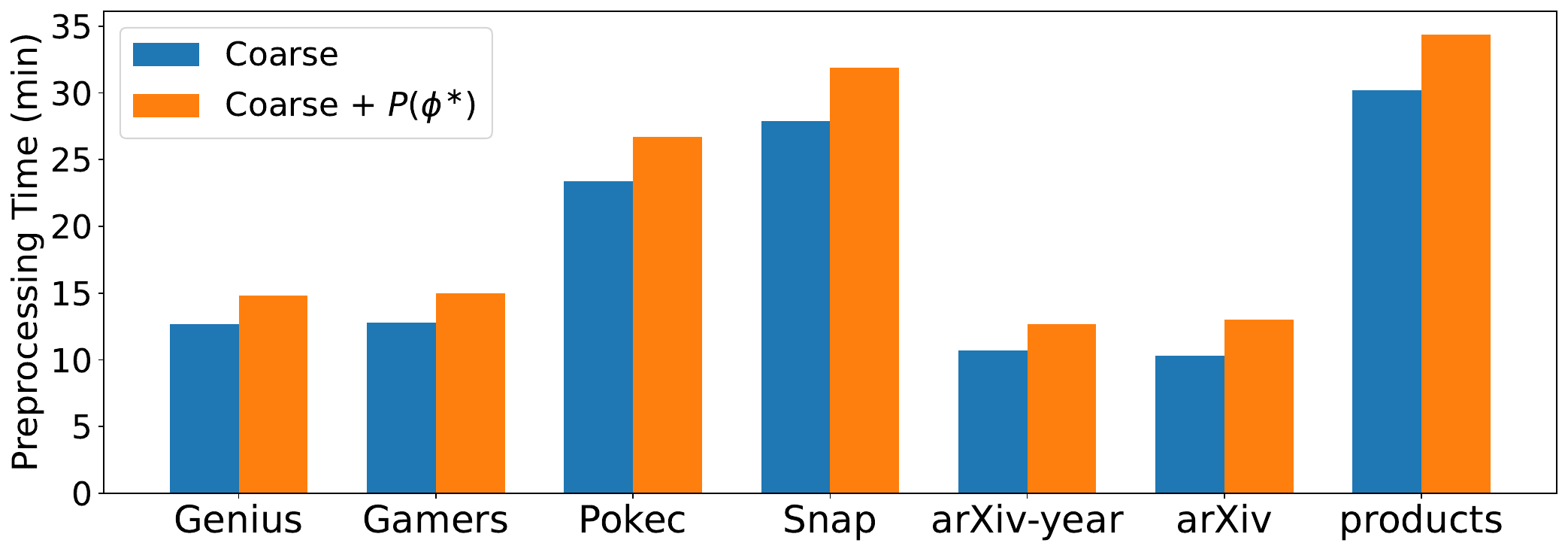}
    }\\
    \subfloat[Preprocessing time: UGC vs. UGC+ACE.]{
    \label{fig:preprocessing-full-UGC}
    \includegraphics[width=0.495\linewidth]{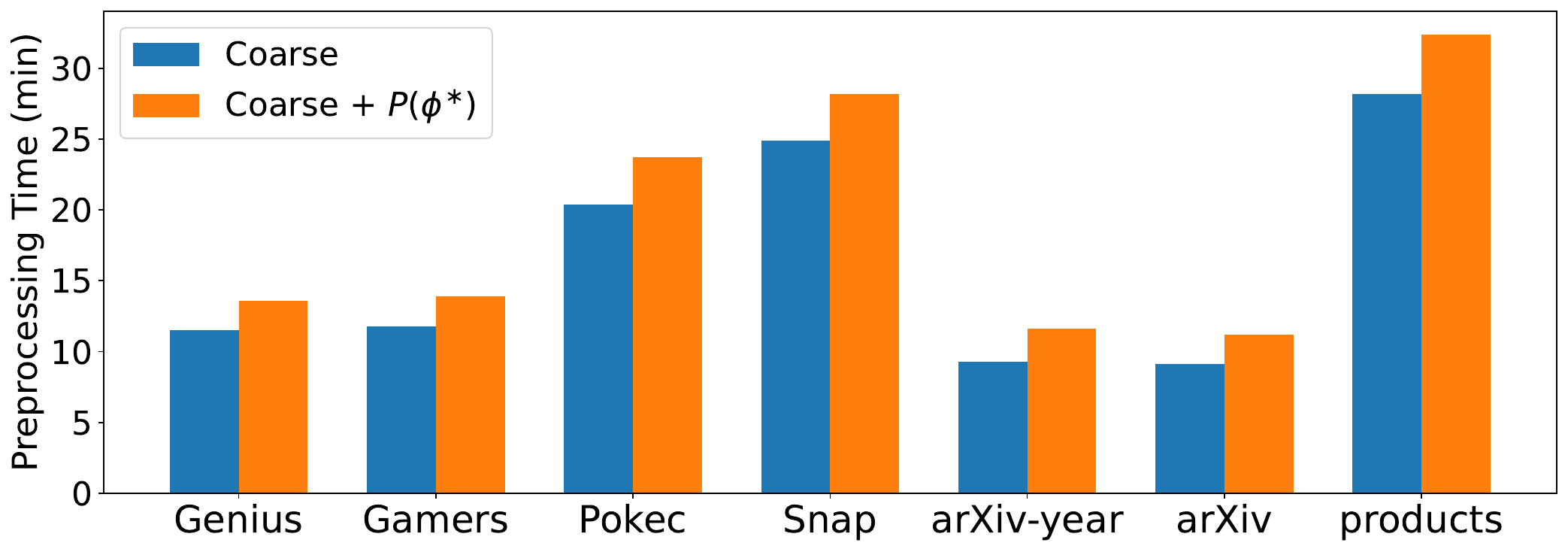}
    }
    \subfloat[Preprocessing time: SGBGC vs. SGBGC+ACE.]{
    \label{fig:preprocessing-full-sgbgc}
    \includegraphics[width=0.495\linewidth]{figures/sgbgc-preprocess.pdf}
    }
  \caption{Preprocessing time of each coarsening training pipelines and those with our ACE across diverse benchmarks.}
  \label{fig:preprocessing-full}
\end{figure*}
\begin{figure*}[!t]
  \centering
    \subfloat[SCAL training time per epoch.]{
    \label{fig:training-time-full-scal}
    \includegraphics[width=0.495\linewidth]{figures/scal-pokec-training.pdf}
    }
    \subfloat[FGC training time per epoch.]{
    \label{fig:training-time-full-FGC}
    \includegraphics[width=0.495\linewidth]{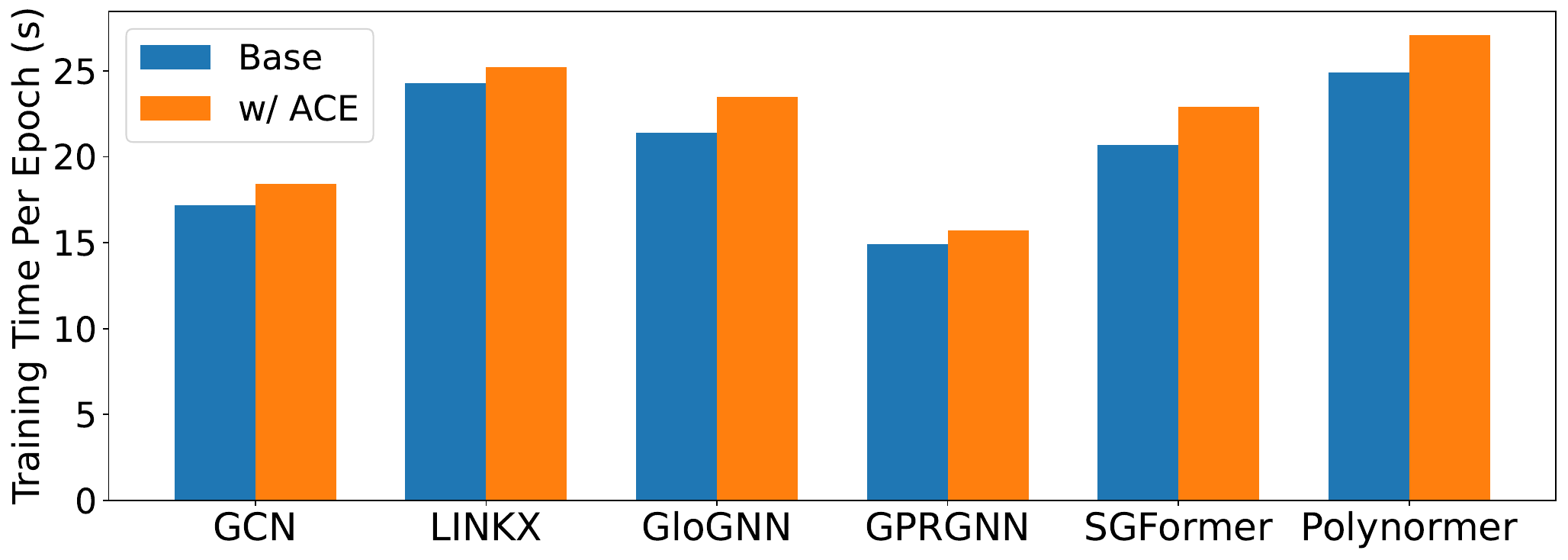}
    }\\
    \subfloat[SGBGC training time per epoch.]{
    \label{fig:training-time-full-sgbgc}
    \includegraphics[width=0.495\linewidth]{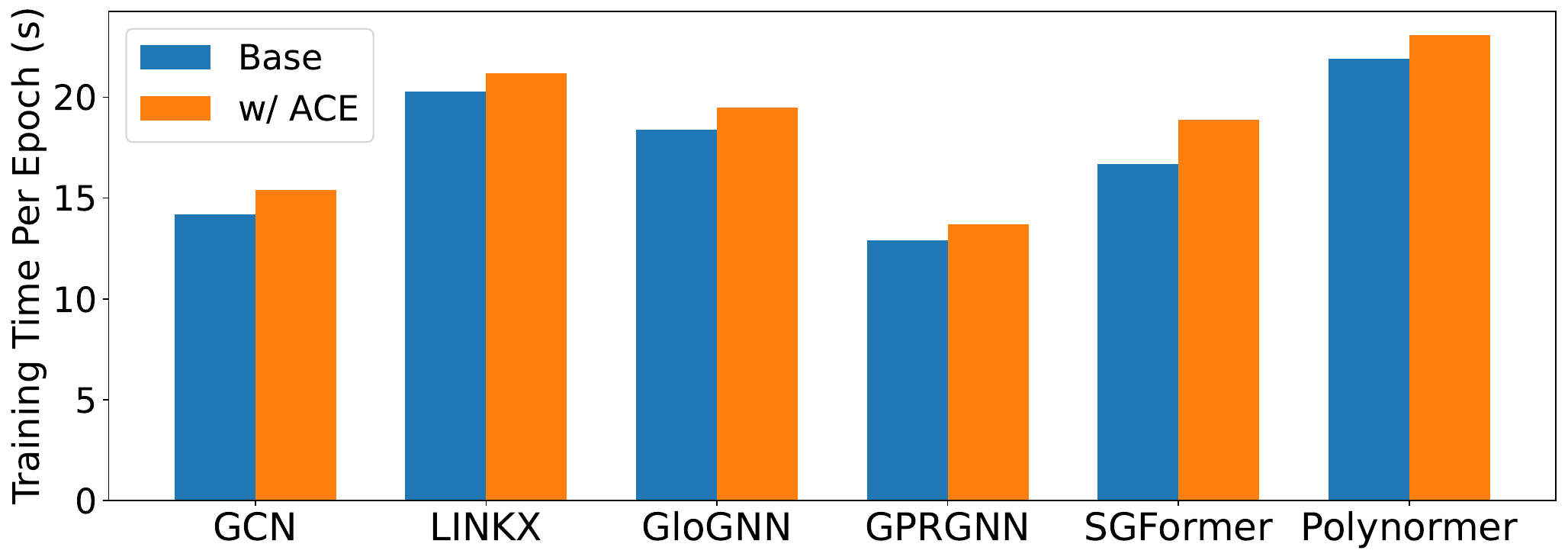}
    }
  \caption{Per-epoch training time of SCAL, FGC, SGBGC, and their ACE-improved versions across diverse backbone GNNs.}
  \label{fig:training-time-full}
\end{figure*}
\begin{figure*}[!t]
  \centering
    \subfloat[SCAL training memory usage.]{
    \label{fig:gpu-memory-full-scal}
    \includegraphics[width=0.495\linewidth]{figures/scal-pokec-memory.pdf}
    }
    \subfloat[FGC training memory usage.]{
    \label{fig:gpu-memory-full-FGC}
    \includegraphics[width=0.495\linewidth]{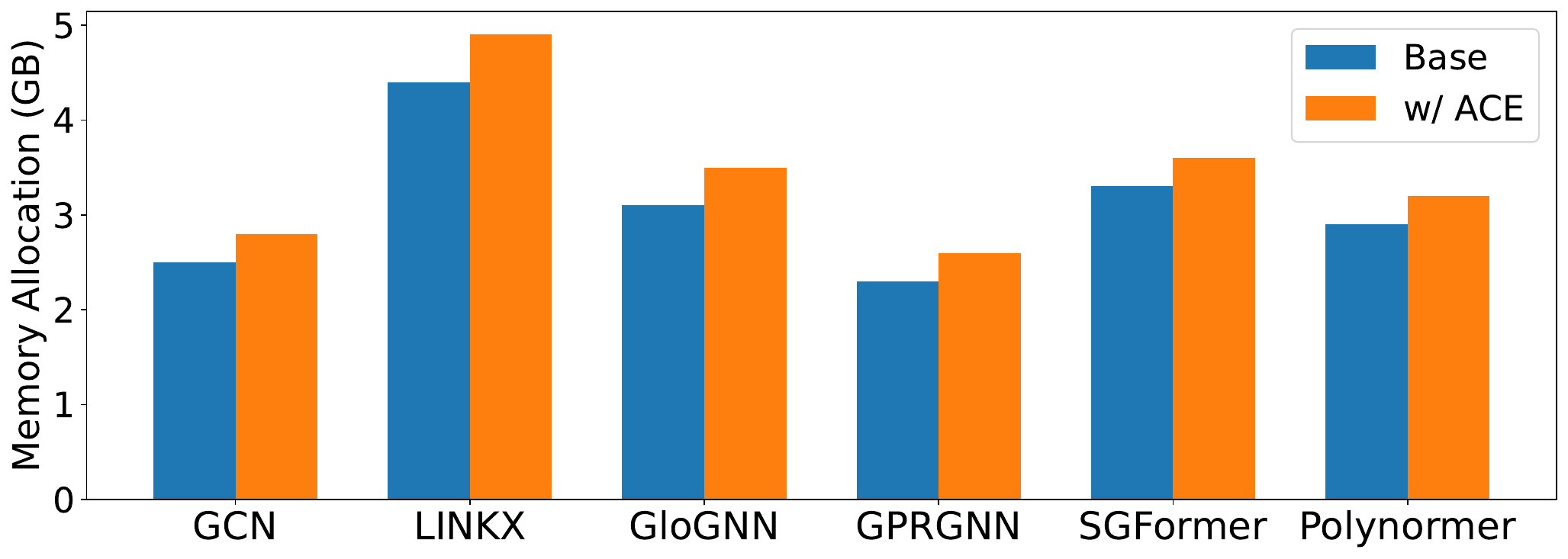}
    }\\
    \subfloat[SGBGC training memory usage.]{
    \label{fig:gpu-memory-full-sgbgc}
    \includegraphics[width=0.495\linewidth]{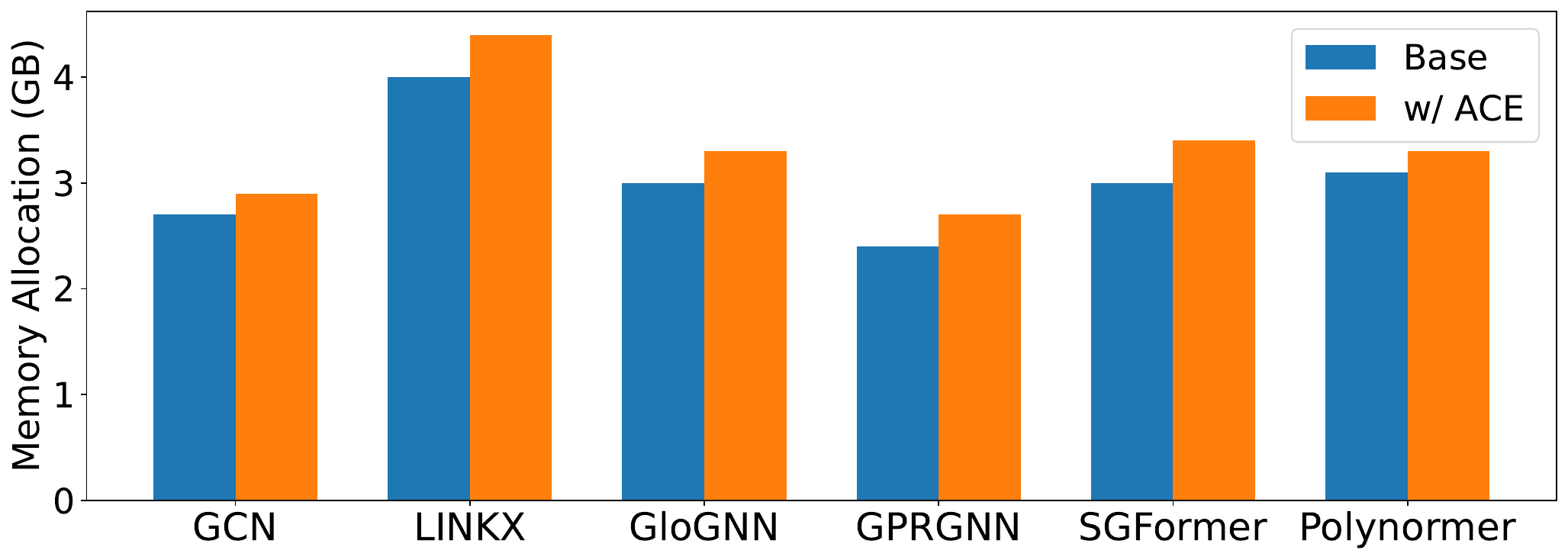}
    }
  \caption{Training GPU memory allocation of SCAL, FGC, SGBGC, and their ACE-improved versions across diverse backbone GNNs.}
  \label{fig:gpu-memory-full}
\end{figure*}
\begin{figure}[!t]
  \centering
    \subfloat[GCN as backbone.]{
    \label{fig-1}
    \includegraphics[width=0.475\linewidth]{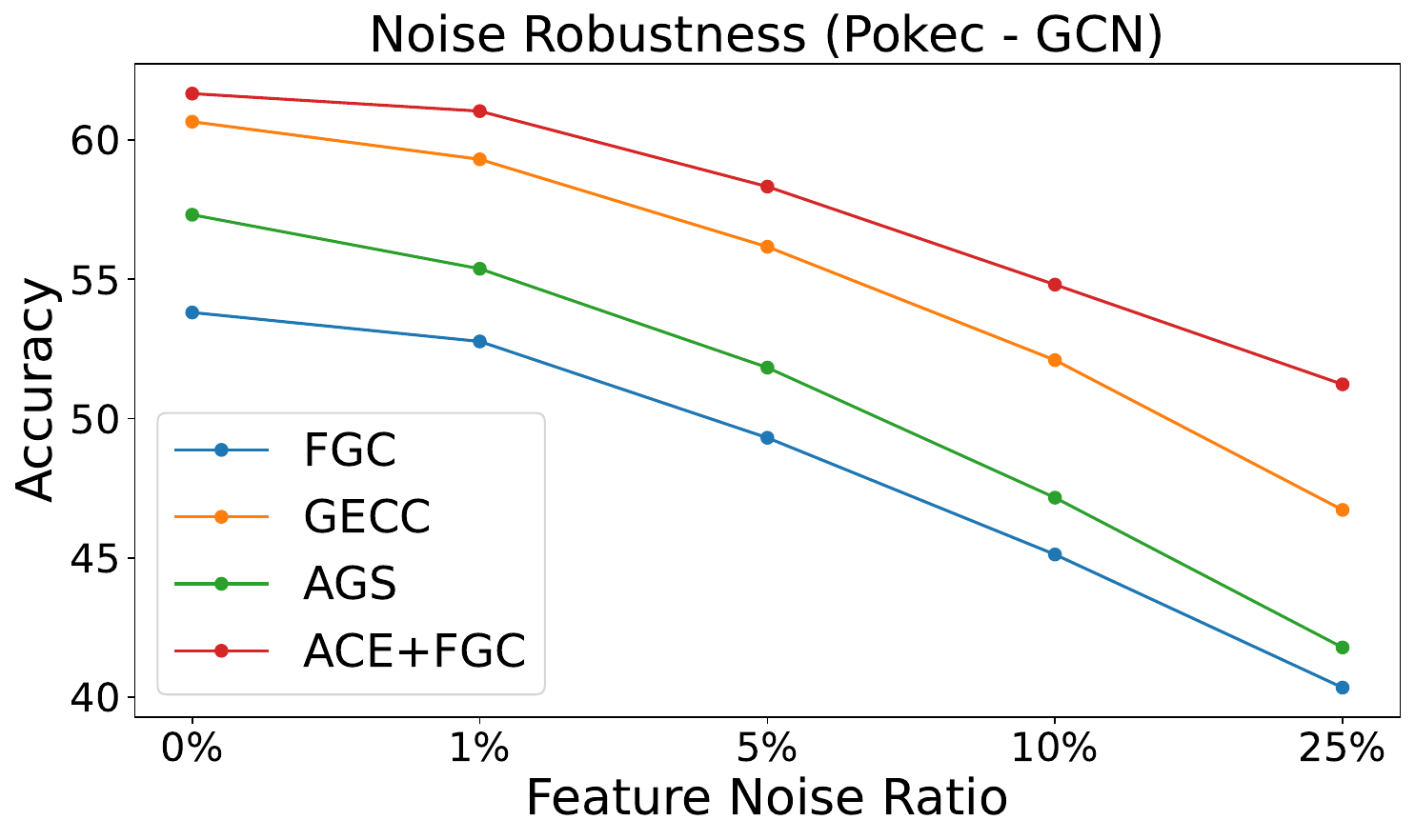}
    }
    \subfloat[GPRGNN as backbone.]{
    \label{fig-2}
    \includegraphics[width=0.475\linewidth]{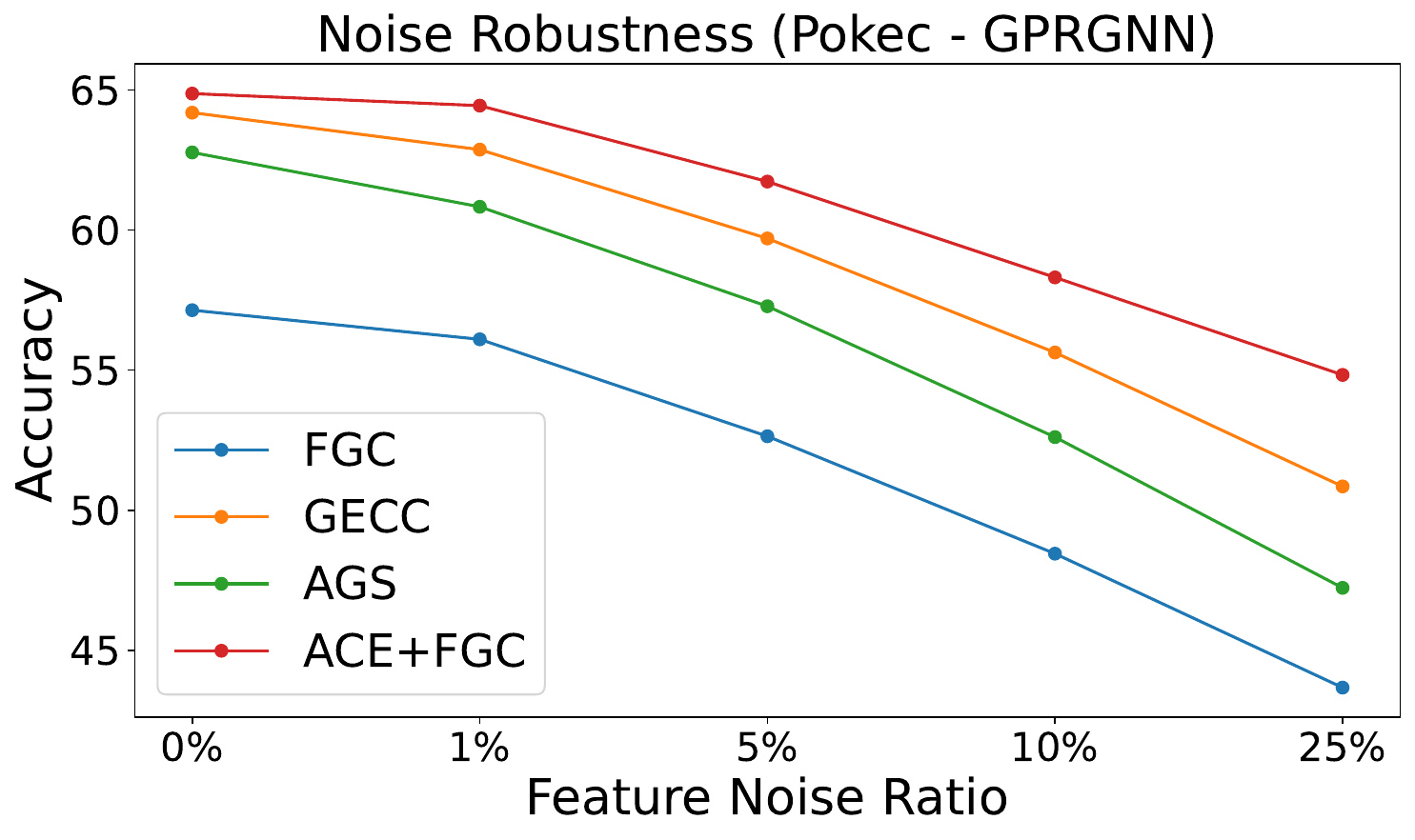}
    }
  \caption{Noise robustness evaluation on the Pokec dataset. Experiments are conducted using GCN and GPRGNN backbones with a 10\% graph reduction ratio. We report performance under feature perturbations, where Gaussian noise is added to randomly sampled node features in the original graph, with sampling ratios of \{0\%, 1\%, 5\%, 10\%, 25\%\}.}
  \label{fig:noiserobust}
\end{figure}

\end{document}